\documentclass{article}
\usepackage{iclr2027_conference,times}
\iclrfinalcopy 

\usepackage{amsmath,amsfonts,bm}

\def\eqref#1{equation~\ref{#1}}

\def\1{\bm{1}}

\DeclareMathAlphabet{\mathsfit}{\encodingdefault}{\sfdefault}{m}{sl}
\SetMathAlphabet{\mathsfit}{bold}{\encodingdefault}{\sfdefault}{bx}{n}

\usepackage{hyperref}
\usepackage{url}
\hypersetup{hidelinks,hypertexnames=false,
  pdftitle={Grab a Coffee: Future-Aware Guidance for Discrete Diffusion with Compiled Objectives},
  pdfauthor={Hua (Edward) XU, Dongxin Li, Gwen Yidou-Weng, Guy Van den Broeck, Wei Wang, Anji Liu}}
\usepackage{booktabs}
\usepackage{colortbl}
\usepackage{amsthm}
\usepackage{algorithm}
\usepackage[noend]{algpseudocode}
\usepackage{array}
\usepackage{float}
\usepackage{stfloats}
\fnbelowfloat
\usepackage{graphicx}
\usepackage{longtable}
\usepackage{mathtools}
\usepackage{microtype}
\usepackage{multirow}
\usepackage{placeins}
\usepackage{tabularx}
\usepackage{xcolor}
\usepackage{tikz}
\usetikzlibrary{arrows.meta,calc,positioning}

\definecolor{BaseColor}{HTML}{26366F}
\definecolor{RewardColor}{HTML}{1A9B73}
\definecolor{OursColor}{HTML}{426CB5}
\definecolor{CompareColor}{HTML}{EF7E22}
\definecolor{MutedColor}{HTML}{8290AA}
\definecolor{WarningColor}{HTML}{E87891}
\definecolor{PanelFill}{HTML}{F1F3FB}
\definecolor{DraftColor}{HTML}{A95D20}
\definecolor{oursrow}{RGB}{240,244,252}
\definecolor{DreamHeader}{HTML}{D96E1B}
\definecolor{LLaDAHeader}{HTML}{426CB5}
\definecolor{HeaderFill}{HTML}{F1F3F8}
\definecolor{TargetHeader}{HTML}{26366F}
\definecolor{QualityHeader}{HTML}{66728E}
\definecolor{BudgetAmber}{HTML}{D96E1B}
\definecolor{BudgetCyan}{HTML}{7898D0}
\definecolor{BudgetGray}{HTML}{8290AA}
\definecolor{BudgetPurple}{HTML}{8B78D2}
\definecolor{BudgetBrown}{HTML}{397B69}
\definecolor{Fig3Coffee}{HTML}{607BAE}
\definecolor{Fig3Base}{HTML}{778198}
\definecolor{Fig3CDD}{HTML}{78A99A}
\definecolor{Fig3CDM}{HTML}{C78D9D}
\definecolor{Fig3Amber}{HTML}{D4A16F}
\definecolor{Fig3Cyan}{HTML}{91A9CE}
\definecolor{Fig3Gray}{HTML}{A0A6B0}
\definecolor{Fig3Purple}{HTML}{9A8DBA}
\definecolor{Fig3Teal}{HTML}{6E9B8D}
\definecolor{Fig3Dream}{HTML}{C38C62}
\definecolor{Fig3LLaDA}{HTML}{7189B0}

\newcommand{\method}{\textsc{Coffee}}

\newcolumntype{C}[1]{>{\centering\arraybackslash}p{#1}}

\newtheorem{theorem}{Theorem}[section]
\newtheorem{proposition}[theorem]{Proposition}
\newtheorem{lemma}[theorem]{Lemma}
\newtheorem{corollary}[theorem]{Corollary}

\usepackage{caption}
\title{Grab a Coffee: Future-Aware Guidance for
Discrete Diffusion with Compiled Objectives}
\author{\textbf{Hua (Edward) XU}$^{1,*}$, \textbf{Dongxin Li}$^{1,*}$, \textbf{Gwen Yidou-Weng}$^{2}$,\\
\textbf{Guy Van den Broeck}$^{2}$, \textbf{Wei Wang}$^{1}$, \textbf{Anji Liu}$^{3}$\\[0.5em]
{\normalfont\small $^{1}$Data Science and Analytics Thrust,}\\
{\normalfont\small The Hong Kong University of Science and Technology (Guangzhou), Guangzhou, China}\\
{\normalfont\small $^{2}$Department of Computer Science, University of California, Los Angeles, USA}\\
{\normalfont\small $^{3}$School of Computing, National University of Singapore, Singapore}\\[0.3em]
{\normalfont\small $^{*}$Equal contribution.}}

\begin{document}
\maketitle
\lhead{Preprint}

\begin{abstract}
Discrete diffusion models generate sequences by iteratively resolving multiple
tokens in parallel, offering a flexible alternative to left-to-right
generation. However, guiding this process with a sequence-level objective is
difficult because the value of one unresolved token depends on the other tokens
with which it can form a high-reward sequence. Enumerating all such completions
makes the whole guidance computation grow exponentially with the number of
unresolved positions. We introduce \method{}, a plug-and-play framework that
avoids this enumeration by separating sequence dependence from the objective.
At each diffusion step, a target-free carrier absorbs the marginal token
distributions predicted by the denoiser to construct a joint model over the
unresolved tokens, while a compiled finite-state model records how their
combinations affect the sequence-level preference. Pairing their states allows
\method{} to transfer global preferences to unresolved positions and sample a
clean reconstruction without retraining the diffusion model. The same
framework supports explicit hard constraints and learned soft objectives. We
evaluate \method{} across multiple symbolic, language, and biological benchmarks,
where it achieves strong control results with task-dependent quality and
diversity trade-offs.
By making objectives available to inference rather than only evaluation,
\method{} brings joint conditioning, completion-weighted guidance,
and optimization-based constraints into pretrained neural generation,
showing the potential of neural-symbolic methods in diffusion guidance.
\end{abstract}

\section{Introduction}
\label{sec:introduction}

Discrete diffusion models offer a flexible approach to generating
language and biological sequences
\citep{nie2025largelanguagediffusionmodels,ye2025dream7bdiffusionlarge,yang2026d3lm}.
By repeatedly reconstructing tokens from corrupted inputs, they allow
multiple positions to be predicted together using context from either
side of the sequence.
In many applications, however, generating plausible samples is only
a starting point: the output should also satisfy a constraint, express
a desired attribute, or achieve a high sequence-level reward.\footnote{In this paper, we are using ``reward'', ``condition'', ``constraint'' and ``preference'' interchangeably.}
Inference-time guidance offers a way to introduce these objectives
without retraining the pretrained generator.
The challenge is to turn a preference over complete sequences into
useful guidance while the current state is still partially corrupted.
Consider the masked sentence
\texttt{[MASK] [MASK] is in [MASK]} in Fig.~\ref{fig:fig1}, with a preference
for the phrase \texttt{New York}.
The preference concerns the term \texttt{New York} together,
but a scorer itself without knowing joint distribution of tokens
does not specify how the unresolved
tokens should be chosen jointly.
Directly sampling from the distributions given by the denoiser
may
produce mismatched combinations such as \texttt{New Kong}, failing to capture cross-token dependence.
The scorer's preference may also affect the distributions at other positions,
ideally coupling the preferred term with \texttt{USA},
though the scorer itself may not be able to associate them together.
Therefore, using such reward to guide the generation process
requires a joint model which can couple these choices and propagate the phrase
preference to the country prediction through the association
between \texttt{New York} and \texttt{USA}, even though the scorer
does not directly reward the country.
Effective guidance therefore needs to account for both
\emph{how token choices depend on one another} and
\emph{how their combinations affect the sequence-level objective}.

More generally, let $x_t$ denote the current diffusion state and
$x_0$ a possible clean reconstruction.
Given a neutral, unguided distribution
$q_0(x_0\mid x_t)$ which models the joint token distribution
before imposing condition $C$,
and a nonnegative sequence
weight $W_C(x_0)$ indicating sequence-level preference,
for a candidate token $u$ at an unresolved position $i$,
we consider the guided reconstruction distribution
\begin{equation}
q_C(x_0^i=u\mid x_t)
\propto
q_0(x_0^i=u\mid x_t)\,
\mathbb{E}_{q_0}\!\left[
W_C(X_0)\mid x_t,\;X_0^i=u
\right].
\label{eq:intro-completion-guidance}
\end{equation}
Our goal is to sample from $q_C$, which brings two computation challenges.
\textbf{(1)} The expectation depends on how unresolved tokens occur
together, i.e. the joint distribution $q_0$,
whereas a single denoiser can only provide the factorized approximation
$\prod_{i\in\mathcal M_t}p_\theta(x_0^i\mid x_t)$
that does not by itself specify these dependencies.
With unresolved
positions $\mathcal M_t$, naively computing the joint
requires
$|\mathcal V|^{|\mathcal M_t|}$ operations
over vocabulary $\mathcal V$
.
\textbf{(2)}
Even with access to $q_0(x_0 \mid x_t)$,
evaluating the expectation of $W_C$ still requires scoring the sequence-level preference
over all possible completions separately,
which also requires
$|\mathcal V|^{|\mathcal M_t|}$ operations (Fig.~\ref{fig:fig1}(a)).
The problem is thus not simply to score a completed sequence,
but \textbf{to obtain coupling token choices} per diffusion step and \textbf{to evaluate sequence-level preferences}
while the current state is
still partially corrupted.

\begin{figure}
    \centering
    \includegraphics[width=1\linewidth]{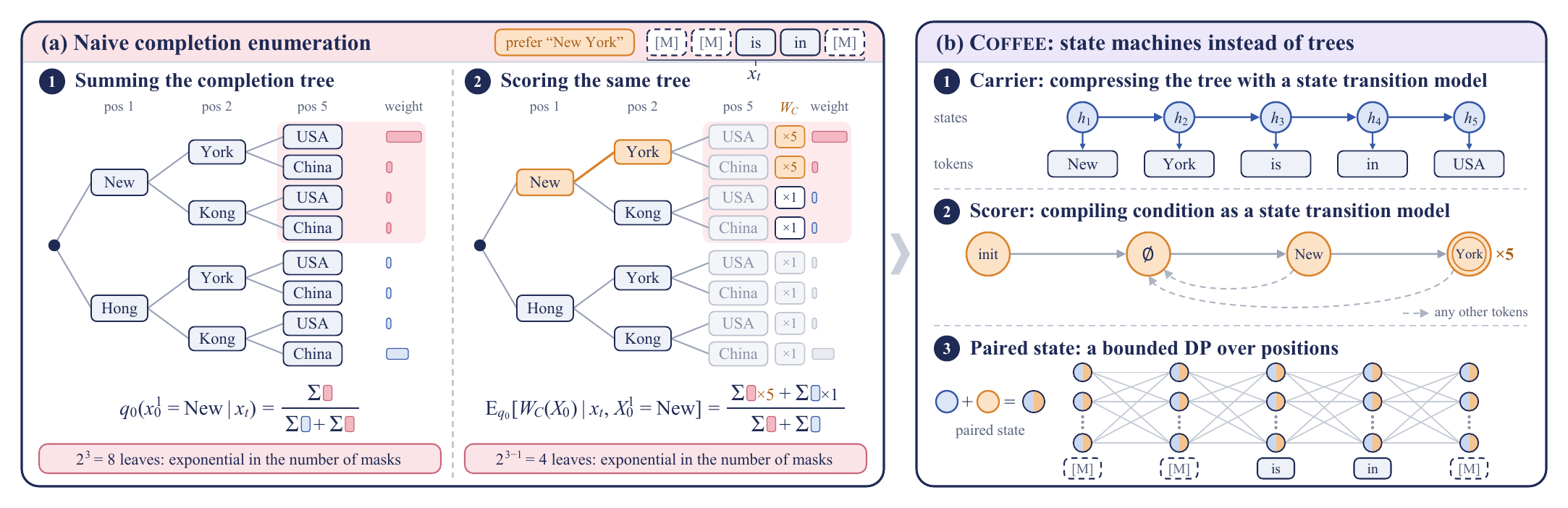}
    \caption{
    \textbf{Sequence-level guidance combines preferences with a model of possible
completions.} (a) Each root-to-leaf path specifies a clean reconstruction
with neutral completion weights given by $q_\theta$.
With two token choices at
each of three unresolved positions, explicit enumeration produces $2^3$
leaves. 
Calculating quantities in 
Eq.~\ref{eq:intro-completion-guidance}
with such way is therefore hard with exponential complexity.
(b)
\method{} uses state machines to 
represent token dependence
and record pattern progress. Their product permits histories reaching
the same state at the same position to share a continuation calculation,
replacing path-by-path enumeration within the compiled model.}
    \label{fig:fig1}
\end{figure}

Existing diffusion-guidance methods approach control through several routes,
including attribute estimates on noisy inputs
\citep{nisonoff2025guidance,schiff2025simple} and clean-sample search that
evaluates rewards on completed sequences and iteratively refines them
\citep{phunyaphibarn2026csmc}. These methods provide ways to use sequence-level
objectives during generation. A separate line of work restores the
cross-position dependence absent from factorized denoiser outputs. CoDD, for
example, augments parallel diffusion predictions with a joint model
\citep{li2026codd}. Together, these directions provide the two ingredients
needed by the example above: a sequence-level preference and a model of
how token choices depend on each other.

We then introduce \method{}, a plug-and-play inference-time guidance
framework based on graphical models. A target-free graphical model,
called the \emph{dependence carrier},
is used to capture the joint dependency with diffusion outputs, and
a \emph{compiled objective} expresses the sequence preference through
local state transitions and weights, and is able to score the sequence-level preference on a corrupted sentence.
Here, our graphical model summarizes the information needed by the remaining computation into latent states,
and hence we are able to run dynamic programming (DP) algorithms on the latent states,
which avoids explicitly enumerating all the complete sequences.
With proper parameterization and model construction, the resulting cost will be polynomial in the problem size,
instead of exponential as we previously discussed.
At each diffusion step, \method{} constructs this joint model from the
current denoiser predictions, computes the guided probabilities, and samples a
clean reconstruction for the diffusion sampler's existing update rule.
The same framework can support both hard constraints and soft preferences by compiling
the objective or a fitted structured surrogate
into different forms of graphical models, which can enable us to perform exact reasoning and probability inference.
It is a plug-and-play method, leaving pretrained diffusion
parameters unchanged and saving time for compute-heavy parameter tuning.




\section{Method}
\label{sec:method}

\subsection{Problem Formulation}
\label{sec:problem}

At a fixed diffusion step, we seek a clean reconstruction that is consistent
with the denoiser's predictions and preferred by a sequence-level condition.
Let $x_t=(x_t^1,\ldots,x_t^L)$ denote the current diffusion state and
$x_0=(x_0^1,\ldots,x_0^L)$ a possible clean reconstruction over vocabulary
$\mathcal V$.
For absorbing-mask denoising, $\mathcal O_t$ contains
all the positions that are not masked,
and the unresolved
positions form $\mathcal M_t=[L]\setminus\mathcal O_t$, where
$[L]=\{1,\ldots,L\}$.
In one denoising step,
we can start from the predicted token distributions supplied by one forward pass through
the frozen denoiser,
keeping observed tokens fixed and then defining the diffusion evidence
\begin{equation}
e_{t,i}(u)=
\begin{cases}
\mathbf 1[u=x_t^i], & i\in\mathcal O_t,\\
p_\theta(x_0^i=u\mid x_t), & i\in\mathcal M_t,
\end{cases}
\qquad u\in\mathcal V,
\label{eq:evidence}
\end{equation}
where $p_\theta(x_0^i=u\mid x_t)$ is the frozen denoiser's clean-token
conditional at position $i$.
To perform reward-guided generation, these token-wise marginals need to be used to construct the joint distribution
$q_0(x_0\mid x_t)$ over clean reconstructions
and then impose the preference $C$ through a nonnegative sequence weight $W_C$.
Putting everything together gives us
\begin{equation}
q_C(x_0\mid x_t)=\frac{q_0(x_0\mid x_t)W_C(x_0)}{Z_C(x_t)}.
\label{eq:guided-law}
\end{equation}
Its normalizer sums over completions consistent with the observed tokens and,
because $q_0(\cdot\mid x_t)$ is normalized on that support, is the expected
condition weight under the neutral joint distribution:
\begin{equation}
Z_C(x_t)
=\sum_{x_0^{\mathcal M_t}}q_0(x_0\mid x_t)W_C(x_0)
=\mathbb E_{x_0\sim q_0(\cdot\mid x_t)}[W_C(X_0)]
,\qquad x_0^{\mathcal O_t}=x_t^{\mathcal O_t}
.
\label{eq:completion-partition}
\end{equation}
Different values of $W_C$ could have different interpretations. For example, setting $W_C\equiv1$ leaves $q_0$ unchanged, while a hard condition uses
$W_C(x_0)=\mathbf 1[x_0\models C]$ to exclude candidates that violate $C$,
where $x_0\models C$ means that the
complete sequence satisfies $C$. A soft condition uses
$W_C(x_0)=\exp\{\lambda R_C(x_0)\}$, where $R_C$ is a sequence-level
reward,
where the strength $\lambda\geq0$ controls the preference for larger
scores.
Specially, when
$W_C(x_0)=q(C\mid x_0)$ is an attribute likelihood, the same
expression gives Bayesian conditioning.

The single-token conditional in Eq.~\ref{eq:intro-completion-guidance} is a
marginal of the joint distribution in Eq.~\ref{eq:guided-law}. Our goal is
to sample that joint reconstruction rather than draw each token independently
from its marginal.
As discussed previously, two main challenges are: \textbf{(1) coupling token choices} despite the
denoiser's factorized output, and \textbf{(2) evaluating the sequence objective}
while the tokens that determine its value remain unresolved.
These challenges block the way of obtaining the quantities needed in
Eq.~\ref{eq:guided-law} and \ref{eq:completion-partition},
resulting in a $\mathcal{O}(|\mathcal V|^{|\mathcal M_t|})$ level complexity (Fig.~\ref{fig:fig1}(a)).
\method{} bypasses the challenges with state-transition graphical
models in which a \emph{dependence carrier} retains information linking token choices,
and a \emph{compiled objective} tracks their effect on the condition (Fig.~\ref{fig:method-overview}(A)).
In the later sections, we will construct the carrier first, then extend its calculation to the objective.

\begin{figure}[thbp]
  \centering
  \includegraphics[width=\linewidth]{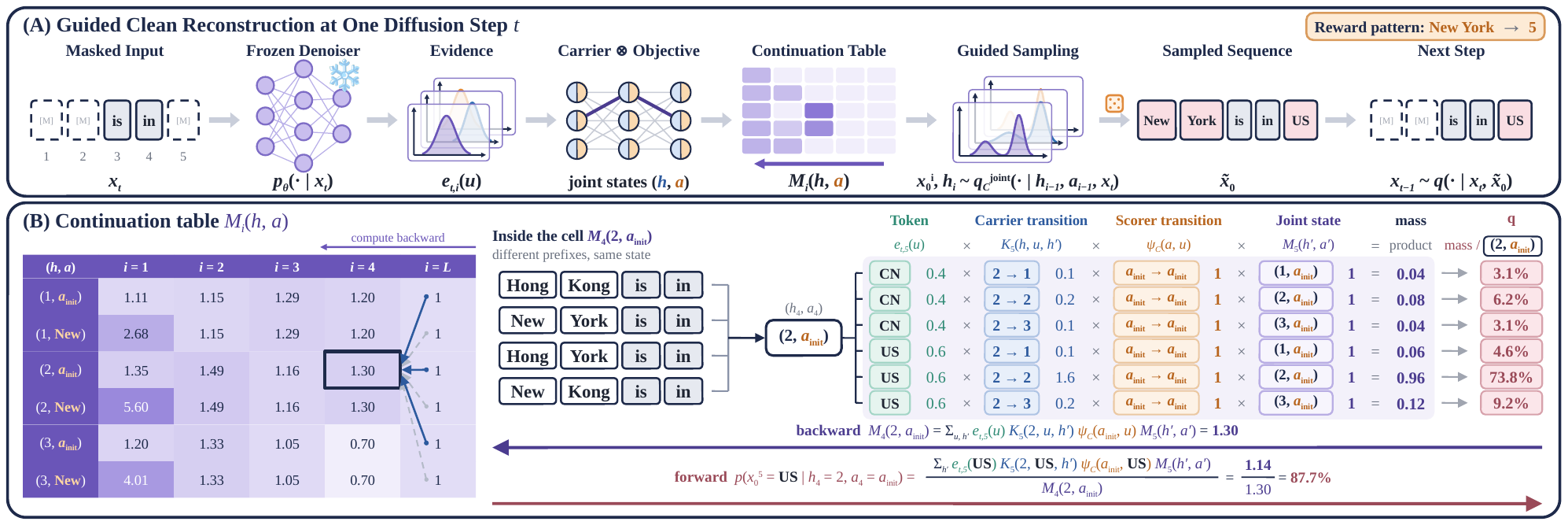}
  \caption{\textbf{Joint states support tractable computation for joint sampling.}
The upper panel shows how \method{} transforms diffusion marginals into guided joint distributions. 
With compiled states, \method{} runs DP over a continuation table
to obtain probability masses, based on which we can get normalized probabilities.
The lower panel uses a toy example to dive into the DP process in the table 
to provide an intuitive understanding of 
Eq.~\ref{eq:product-recurrence} and~\ref{eq:guided-sampling-normalized}.
In the table,
columns index sequence boundaries and rows index joint
states $(h,a)$.
Each $M_i(h,a)$ sums the weights of the remaining token-state paths,
and generation paths sharing the same boundary state can
therefore reuse one entry.
The expanded example sums six consecutive states $(u,h')$ 
to obtain $M_4(2,a_{\mathrm{init}})$.
Normalizing 
\texttt{US}'s probability mass with the previous quantity 
gives the joint probabilities of the country.
}
  \label{fig:method-overview}
\end{figure}

\subsection{Challenge 1: Constructing Neutral Joint Token Distribution}
\label{sec:carrier}

Directly sampling the denoiser distributions gives the fully factorized distribution
$\prod_i e_{t,i}(x_0^i)$, and the most naive way to construct a joint one is to expand one token at a
time while retaining every possible choice, producing a computation tree with
$|\mathcal V|^{|\mathcal M_t|}$ leaves (Fig.~\ref{fig:fig1}(a)). Inspired by
CoDD~\citep{li2026codd}, we are compressing the computation tree with states in
state transition graphical models (called \emph{carriers}), and the
denoiser marginals are used as local evidence absorbed by our carrier to
capture the joint dependency.
For simplicity, here we are using chain structured Markov models to 
explain how \method{} works.
Let $\omega$ be the parameters of our carrier and $h_i\in\mathcal H$ be the carrier state after position $i$,
we use nonnegative local factors
\begin{equation}
K_i(h,u,h';\omega)\geq 0,
\label{eq:carrier-kernel}
\end{equation}
where $K_i(h,u,h';\omega)$ couples token $u$ with a state transition.
We suppress $\omega$ below. With fixed $h_0$, we can use these factors and the denoiser
evidence jointly to construct the neutral joint distribution
\begin{equation}
q_0(x_0,h_{1:L}\mid x_t)
\propto \prod_{i=1}^{L}
e_{t,i}(x_0^i)K_i(h_{i-1},x_0^i,h_i),
\label{eq:carrier-joint}
\end{equation}
\begin{equation}
q_0(x_0\mid x_t)
=\sum_{h_{1:L}}q_0(x_0,h_{1:L}\mid x_t).
\label{eq:carrier-marginal}
\end{equation}
If we fix the carrier state $h_i$ after position $i$, a completion from this state
specifies both the remaining clean tokens $x_0^{i+1:L}$ and the carrier states
$h_{i+1:L}$ connecting them. Its remaining probability weight is the product of the
denoiser evidence and carrier factors at positions $i+1$ through $L$.
Due to the Markov property,
any two
partial paths ending in this state have the same weighted completions.
We therefore sum their weights once for each state, defining the total
continuation weight
\begin{equation}
\begin{aligned}
B_i(h_i)
&=\sum_{x_0^{i+1:L},h_{i+1:L}}
\prod_{j=i+1}^{L}e_{t,j}(x_0^j)K_j(h_{j-1},x_0^j,h_j).
\end{aligned}
\label{eq:carrier-remaining-mass}
\end{equation}
Here, $h_i$ is fixed and the sum ranges over every assignment to the remaining
tokens and carrier states.
This quantity is the normalizer for the distribution of completions
conditioned on $h_i$ and $x_t$.
To compute such quantity, we can start with considering its successor $h'$, where
the current factor $e_{t,i}(u)K_i(h,u,h')$
multiplies the total continuation weight $B_i(h')$ from the successor state,
\begin{equation}
\begin{aligned}
B_{i-1}(h)&=\sum_{u\in\mathcal V}\sum_{h'\in\mathcal H}
e_{t,i}(u)K_i(h,u,h')B_i(h'),
\qquad B_L(h)=1,
\end{aligned}
\label{eq:carrier-recurrence}
\end{equation}
which gives us a recurrence relationship between $B_i$ and $B_{i-1}$ and can be used to perform DP. Hence, all $B_i$ with $i=1..L$ can be computed in polynomial time.
Working backward to the fixed start state gives
$B_0(h_0)=Z_0(x_t)$, and that is exactly the normalizer of
Equation~\ref{eq:carrier-joint}.
For a reachable state $h$ with $B_{i-1}(h)>0$, we can partition its succeeding
completions by the next token-state pair $(u,h')$, which has a probability weight
$e_{t,i}(u)K_i(h,u,h')B_i(h')$.
To obtain the distribution of $(x_0^i=u,h_i=h')$, we just need to normalize with all the local factors at this state, which gives
\begin{equation}
\begin{aligned}
q_0(x_0^i=u,h_i=h'\mid h_{i-1}=h,x_t)
&\quad=
\frac{e_{t,i}(u)K_i(h,u,h')B_i(h')}
{\sum_{v,g}e_{t,i}(v)K_i(h,v,g)B_i(g)}
\\
&\quad=
\frac{e_{t,i}(u)K_i(h,u,h')B_i(h')}
{B_{i-1}(h)}.
\end{aligned}
\label{eq:carrier-conditional}
\end{equation}
With all $B_i$ computed, we can start at $h_0$ and draw
$(x_0^i, h_i)$ pairs from $i=1$ to $L$ using these conditionals in polynomial time, which is equivalent to sampling from
Equation~\ref{eq:carrier-joint} if we
discard the sampled carrier states from the sampled results.
The carrier is trained without any task-specific preference $C$ and is therefore
target-free.
During generation, the carrier stays fixed while each diffusion
step supplies fresh denoiser evidence. Carrier construction and training details
are in Sec.~\ref{app:reward-training}, while fixed-step sampling details are in
Sec.~\ref{app:formal-setup}.

\subsection{Challenge 2: Scoring Corrupted Sequences without Naive Enumeration}
\label{sec:reward-compiler}

In the previous section we have shown that constructing and then sampling from the joint distribution can be tractable,
and we now try to incorporate the sequence level objective into its continuation calculation.
To draw samples from Eq.~\ref{eq:guided-law}, we should make sure that the $W_C$ won't break the computation complexity,
and this is what Challenge~2 is about.
Using $W_C^{\mathrm{oracle}}(x_0)$ to denote the desired weight, with no
assumed structure,
our goal is to seek a structured surrogate scorer $W_C$ for the same task
signal,
whose structure is compatible with the DP computation structure to preserve the complexity.
The process of finding such surrogate $W_C$ is called \emph{compilation}, and usually, we just need to find the proper structures we need,
and fit the model with the same objectives or distill from a better one.
With Prop.~\ref{prop:compilation-approximation}, with the surrogate structure that will be specified later,
the approximation error of such compilation can be analysed and controlled. Details are specified in Sec.~\ref{app:compilation-approximation}.
\begin{proposition}[Objective approximation under compilation]
\label{prop:compilation-approximation}
Fix $x_t$ and a base distribution $q_0$ with finite support, on which $W_C$ and
$W_C^{\mathrm{oracle}}$ are finite and strictly positive. Exact compilation preserves the guided
law $q_C$. Let $q_C^{\mathrm{oracle}}$ use the same $q_0$ with the oracle weight.
If $|\log W_C(x_0)-\log W_C^{\mathrm{oracle}}(x_0)-c|\leq\epsilon$
throughout this support for some constant $c$ and $\epsilon\geq0$, then
\[
D_{\mathrm{KL}}(q_C^{\mathrm{oracle}}\Vert q_C)\leq\epsilon^2/2.
\]
\end{proposition}

To see the required structure,
we can start with the simplest scorer class,
$W_C^{\mathrm{local}}(x_0)=\prod_i w_C(x_0^i)$, where $w_C(u)\geq0$.
Because each factor concerns only the current token, absorbing it into the
carrier transition gives the same backward structure as
Equation~\ref{eq:carrier-recurrence}, now incorporating the objective weights,
\begin{align}
\overline K_{C,i}(h,u,h')&=K_i(h,u,h')w_C(u),
\label{eq:local-weight-composition}\\
B_{i-1}^{\mathrm{local}}(h)&=\sum_{u\in\mathcal V}\sum_{h'\in\mathcal H}
e_{t,i}(u)\overline K_{C,i}(h,u,h')B_i^{\mathrm{local}}(h').
\label{eq:local-weight-backward}
\end{align}
With $B_L^{\mathrm{local}}(h)=1$, this computes the guided completion weights
and hence the exact conditionals. Each transition only adds one multiplication,
which preserves the dense carrier complexity
of $O(L|\mathcal V|H^2)$, preserving the same time complexity as to sample from the neutral distribution.
However, such scorer cannot
express preferences for structures formed by multiple tokens, for example,
since
$W_C^{\mathrm{local}}(\texttt{AB})=w_C(\texttt A)w_C(\texttt B)
=W_C^{\mathrm{local}}(\texttt{BA})$, it can't even tell orders apart.
More generally, an objective may assign
scores to occurrences of multiple patterns with different lengths, including
repeated and overlapping matches, and they cannot be modeled by such simple class.

Now, our goal is to capture complex patterns while retaining the same DP
structure. To do so, we can try to generalize
Equations~\ref{eq:local-weight-composition} and~\ref{eq:local-weight-backward}.
Notice that the recurrence only requires each objective factor to be
computable from the current state and token
and the scorer shall have the same transition dynamics as the carrier graphical model,
if a finite state $a_i$ can
remember the relevant pattern history, we could therefore extend
$W_C(x_0)$ to $\prod_{i=1}^{L}w_C(a_i)$ with a chain-structured update
$a_i=\delta_{C,i}(a_{i-1},x_0^i)$.
Such an extension enlarges the state space from $h_i$ to $(h_i,a_i)$,
so the same recurrence can operate on the paired states.
With $A$ denoting the number of $a_i$, the overall time complexity will be $\mathcal{O}(LA | V |H^2 )$,
preserving the polynomial time complexity.

For a selected pattern set $\mathcal G_C$, an Aho-Corasick (AC)
automaton~\citep{aho1975efficient} has all the nice desired properties.
Its states
$\mathcal A$ correspond to shared pattern prefixes, with the empty prefix as
$a_{\mathrm{init}}$. After reading $x_0^{1:i}$, the state $a_i$ records the
longest suffix that is also a prefix of a selected pattern. Reading a token
extends the current match or falls back to a shorter suffix, giving the
update $a_i=\delta_{C,i}(a_{i-1},x_0^i)$ from $a_0=a_{\mathrm{init}}$.
For simplicity,
here we define $\delta_{C,i}=\delta_C$ to be position-independent.
Such an automaton can report all selected patterns completed when
reading $u$ from $a$, including overlapping occurrences. For example,
\texttt{New} starts a partial match and the next token \texttt{York}
completes \texttt{New York} (Figure~\ref{fig:fig1}). The transition and
output construction is detailed in Sec.~\ref{app:formal-setup}.

The token-only factor $w_C(u)$ in
Equation~\ref{eq:local-weight-composition} can now be replaced by a
nonnegative local weight $\psi_{C,i}(a,u)$.
Since both models consume the same token, their joint transition
and its weight become
\begin{equation}
(h,a)\xrightarrow{u}\bigl(h',\delta_{C,i}(a,u)\bigr),\qquad
\overline K_{C,i}(h,a,u,h')=K_i(h,u,h')\psi_{C,i}(a,u).
\label{eq:state-weight-composition}
\end{equation}
With the modified transition dynamics,
we can therefore replace $B_i^{\mathrm{local}}(h)$ in
Eq.~\ref{eq:local-weight-backward} by $M_i(h,a)$,
the total continuation
weight from the state pair, and obtain the similar recurrence relationships
with nonnegative terminal factor $\eta_C(a)$
(visualization in Fig.~\ref{fig:method-overview})
\begin{align}
M_L(h,a)&=\eta_C(a),
\label{eq:backward-terminal}\\
M_{i-1}(h,a)&=\sum_{u\in\mathcal V}\sum_{h'\in\mathcal H}
e_{t,i}(u)\overline K_{C,i}(h,a,u,h')
M_i\!\left(h',\delta_{C,i}(a,u)\right).
\label{eq:product-recurrence}
\end{align}
The backward recurrence gives \(M_0(h_0,a_{\mathrm{init}})\), which normalizes the guided joint distribution,
\begin{equation}
q_C^{\mathrm{joint}}(x_0,h_{1:L}\mid x_t)
\quad=\frac{\eta_C(a_L)}{M_0(h_0,a_{\mathrm{init}})}
\prod_{i=1}^{L}e_{t,i}(x_0^i)
\overline K_{C,i}(h_{i-1},a_{i-1},x_0^i,h_i).
\label{eq:product-law}
\end{equation}
Since the tokens uniquely determine the objective-state path,
multiplying its local objective factors gives the sequence-level preference
\begin{equation}
W_C(x_0)=\eta_C(a_L)\prod_{i=1}^{L}
\psi_{C,i}(a_{i-1},x_0^i),
\label{eq:objective-interface}
\end{equation}
and complex patterns are captured with AC state transitions.
With Eq.~\ref{eq:product-law} and \ref{eq:objective-interface}, we can see that the standalone scorer assigns weights to complete sequences, while the joint inference DP process averages these weights over compatible completions and eventually changes the neutral unguided distribution to a more favored one.
In our applications, we just need to train the local factors $\psi(a,u)$ to obtain the compiled objective (Sec.~\ref{app:reward-training}).
The same interface also supports hard conditioning using an
appropriate deterministic objective automaton with unit transition weights and
$\eta_C(a_L)=\mathbf 1[a_L\in\mathcal A_{\mathrm{accept}}]$ to exclude rejecting
paths without requiring a learned scorer.

\subsection{\method{}: Guided Reconstruction}
\label{sec:backward-dp}

We now connect the continuation weights to the expectation in Eq.~\ref{eq:guided-continuation-expectation}.
As discussed previously, using the same diffusion backbone and carrier,
starting from the state $(h,a)$ at the $i$-th token,
$B_i(h)$ is the accumulated probability mass at $h$ state if we run DP without scorer,
while
$M_i(h,a)$ is the total mass at the joint state $(h,a)$, with scorer weights added up to the mass.
Directly taking their ratio could give us
the conditional expectation (details in Eq.~\ref{eq:remaining-objective-expectation})
and get
\begin{equation}
\frac{M_i(h,a)}{B_i(h)}
=\mathbb E_{q_0}\!\left[
\eta_C(a_L)\prod_{j=i+1}^{L}\psi_{C,j}(a_{j-1},x_0^j)
\,\middle|\,h_i=h,x_t;\ a_i:=a\right].
\label{eq:guided-continuation-expectation}
\end{equation}
Based on that, following the same
marginalization as in Sec.~\ref{sec:carrier}, let $a'=\delta_{C,i}(a,u)$.
For a current state with $M_{i-1}(h,a)>0$, the sampling conditional is proportional to
\begin{equation}
\begin{aligned}
&q_C^{\mathrm{joint}}(x_0^i=u,h_i=h'\mid h_{i-1}=h,a_{i-1}=a,x_t)\\
&\quad\propto
\underbrace{e_{t,i}(u)}_{\substack{\text{base denoiser}\\\text{marginal}}}
\quad\underbrace{K_i(h,u,h')B_i(h')}_{\substack{\text{carrier structure}\\\text{and base continuation}}}
\quad\underbrace{\psi_{C,i}(a,u)}_{\substack{\text{current AC}\\\text{objective weight}}}
\quad\underbrace{\frac{M_i(h',a')}{B_i(h')}}_{\substack{\text{expected remaining}\\\text{objective weight}}},
\end{aligned}
\label{eq:forward-sampling}
\end{equation}
which holds for fixed $(h,a)$ on base-supported choices.
Note that $B_i$ can be canceled and hence doesn't need to be computed during sampling.
Using such decomposition, we can further analyze the \texttt{New York} example:
even when
\texttt{New} itself has no reward,
it can be favored through completions in which \texttt{New York} is formed.
Through the carrier’s dependencies, the same preference can also affect positions not directly scored by the objective,
and therefore, result in a higher preference in \texttt{USA}.

To draw samples from the joint distribution, \method{} computes all the $M_i$ values
and, following Eq.~\ref{eq:carrier-conditional}, starts at
$(h_0,a_{\mathrm{init}})$ and samples from $i=1$ to $L$ using
\begin{equation}
q_C^{\mathrm{joint}}(x_0^i=u,h_i=h'\mid h_{i-1}=h,a_{i-1}=a,x_t)
\quad=\frac{e_{t,i}(u)\overline K_{C,i}(h,a,u,h')M_i(h',a')}
{M_{i-1}(h,a)},
\label{eq:guided-sampling-normalized}
\end{equation}
where $\overline K_{C,i}$ is defined in Eq.~\ref{eq:state-weight-composition}
and the denominator sums the numerator over all choices by
Eq.~\ref{eq:product-recurrence}. After drawing $(u,h')$, set
$\widetilde x_0^i=u$ and update the state to $(h',a')$.
Once a complete
reconstruction is obtained, we pass it to the unchanged diffusion update rule
with original noise scheduling $q_{host}$
\begin{equation}
\widetilde x_0\sim q_C(\,\cdot\mid x_t),\qquad
x_{t-1}\sim q_{\mathrm{host}}(\,\cdot\mid x_t,\widetilde x_0),
\label{eq:host-integration}
\end{equation}
and repeat the algorithm in next diffusion steps to eventually obtain clean samples
(Algorithm~\ref{alg:coffee-update}).

\label{sec:extension}
Beyond guided sampling, the finite-state compilation and DP structure
allow \method{} to draw on the extensive algorithmic literature on weighted
automata~\citep{mohri2009weighted}, opening a route to richer queries that
combine optimization with probabilistic inference. Our Dyck repair task
illustrates this potential: we first use a
bounded-stack state graph to retain all valid, prefix-preserving paths with
the fewest token changes within the declared support. This optimal set then
acts as a hard objective: on the retained graph, the same continuation-weight
recurrence and sampling rule combine denoiser evidence with carrier dependence.
Thus, optimization determines the conditioning event, and our joint inference
samples among equally optimal repairs (Sec.~\ref{sec:symbolic-results}).

\begin{table}[h]
\centering
\fontsize{9.0}{9.6}\selectfont
\setlength{\tabcolsep}{0.65pt}
\renewcommand{\arraystretch}{0.88}
\setlength{\aboverulesep}{0pt}
\setlength{\belowrulesep}{0.15ex}
\caption{
    This table shows the generation results under full diffusion budgets.
The warm
\textcolor{DreamHeader}{Dream} and blue \textcolor{LLaDAHeader}{LLaDA} headers denote the two backbones
and the black headers denote biological tasks. \method{} surpasses almost all the baselines on both control and quality metrics over different benchmarks.
}
\label{tab:main-results}
\label{tab:main-symbolic-language}
\label{tab:main-biology-compact}
\label{tab:biology}
\label{tab:protein-budgets}
\resizebox{\textwidth}{!}{%
\begin{tabular}{@{}r||*{8}{c}||*{8}{c}||*{8}{c}@{}}
\toprule
\rowcolor{HeaderFill}
 & \multicolumn{2}{c}{\tiny\textcolor{DreamHeader}{\textbf{Dyck repair}}} & \multicolumn{2}{c}{\tiny\textcolor{DreamHeader}{\textbf{Sudoku}}} & \multicolumn{2}{c}{\tiny\textcolor{DreamHeader}{\textbf{CommonGen}}} & \multicolumn{2}{c||}{\tiny\textcolor{DreamHeader}{\textbf{RTP}}} & \multicolumn{2}{c}{\tiny\textcolor{LLaDAHeader}{\textbf{Dyck repair}}} & \multicolumn{2}{c}{\tiny\textcolor{LLaDAHeader}{\textbf{Sudoku}}} & \multicolumn{2}{c}{\tiny\textcolor{LLaDAHeader}{\textbf{CommonGen}}} & \multicolumn{2}{c||}{\tiny\textcolor{LLaDAHeader}{\textbf{RTP}}} & \multicolumn{2}{c}{\tiny\textcolor{black}{\textbf{DeepSTARR}}} & \multicolumn{2}{c}{\tiny\textcolor{black}{\textbf{APARENT}}} & \multicolumn{2}{c}{\tiny\textcolor{black}{\textbf{K562}}} & \multicolumn{2}{c}{\tiny\textcolor{black}{\textbf{Protein}}}\\[-0.45ex]
\rowcolor{HeaderFill}
\raisebox{0.55ex}[0pt][0pt]{\tiny\textbf{Method}} & {\fontsize{4.5}{4.8}\selectfont Control$\uparrow$} & {\fontsize{4.5}{4.8}\selectfont Quality$\downarrow$} & {\fontsize{4.5}{4.8}\selectfont Control$\uparrow$} & {\fontsize{4.5}{4.8}\selectfont Quality$\downarrow$} & {\fontsize{4.5}{4.8}\selectfont Control$\uparrow$} & {\fontsize{4.5}{4.8}\selectfont Quality$\downarrow$} & {\fontsize{4.5}{4.8}\selectfont Control$\uparrow$} & {\fontsize{4.5}{4.8}\selectfont Quality$\downarrow$} & {\fontsize{4.5}{4.8}\selectfont Control$\uparrow$} & {\fontsize{4.5}{4.8}\selectfont Quality$\downarrow$} & {\fontsize{4.5}{4.8}\selectfont Control$\uparrow$} & {\fontsize{4.5}{4.8}\selectfont Quality$\downarrow$} & {\fontsize{4.5}{4.8}\selectfont Control$\uparrow$} & {\fontsize{4.5}{4.8}\selectfont Quality$\downarrow$} & {\fontsize{4.5}{4.8}\selectfont Control$\uparrow$} & {\fontsize{4.5}{4.8}\selectfont Quality$\downarrow$} & {\fontsize{4.5}{4.8}\selectfont Control$\uparrow$} & {\fontsize{4.5}{4.8}\selectfont Quality$\downarrow$} & {\fontsize{4.5}{4.8}\selectfont Control$\uparrow$} & {\fontsize{4.5}{4.8}\selectfont Quality$\downarrow$} & {\fontsize{4.5}{4.8}\selectfont Control$\uparrow$} & {\fontsize{4.5}{4.8}\selectfont Quality$\downarrow$} & {\fontsize{4.5}{4.8}\selectfont Control$\uparrow$} & {\fontsize{4.5}{4.8}\selectfont Quality$\downarrow$}\\[-0.25ex]
\midrule
Base & 3.0 & 13.5 & 60.4 & 0.1 & 38.6 & 10.1 & 53.5 & 5.2 & 7.5 & 14.0 & 20.4 & 0.3 & 40.9 & \textbf{14.4} & 43.3 & 5.3 & -0.4 & 728.3 & 2.3 & 728.3 & 44.3 & 728.3 & 11.0 & 1.7\\
\addlinespace[0.4pt]
CDD & \underline{46.9} & 14.8 & 99.3 & 0.0 & 46.1 & 30.4 & 55.7 & 17.2 & \underline{62.5} & 14.7 & 90.7 & 0.0 & 40.9 & 18.8 & \underline{81.7} & 13.1 & 3.5 & \underline{1.2} & 7.9 & 682.1 & 83.7 & 70.1 & 10.5 & 1.7\\
\addlinespace[0.4pt]
CDM & 0.2 & \underline{11.7} & \textbf{100.0} & \textbf{0.0} & \underline{57.2} & 33.8 & 65.6 & 5.4 & 7.7 & 13.8 & \textbf{100.0} & \textbf{0.0} & 41.5 & \underline{14.5} & 54.3 & \textbf{4.9} & 0.6 & 429.1 & 10.1 & 691.9 & 72.4 & 250.9 & 10.2 & 1.7\\
\addlinespace[0.4pt]
D-CBG & 14.9 & 14.3 & \underline{99.8} & \underline{0.0} & 39.4 & 10.4 & 62.2 & \underline{5.0} & 23.1 & 14.5 & \underline{98.7} & 0.0 & 30.7 & 32.7 & 72.4 & 7.1 & -0.5 & 1799.7 & 2.2 & 1492.4 & \textbf{100.0} & \underline{1.2} & 11.0 & 1.7\\
\addlinespace[0.4pt]
DG-TAG & \textbf{100.0} & 14.9 & \textbf{100.0} & \textbf{0.0} & 29.7 & \underline{9.8} & \textbf{87.9} & 7.2 & \textbf{100.0} & 14.9 & \textbf{100.0} & \textbf{0.0} & 32.9 & 14.7 & 66.9 & 5.4 & \underline{3.6} & 2.6 & \textbf{24.3} & \textbf{6.3} & 97.2 & 3.4 & \underline{14.5} & \underline{1.5}\\
\addlinespace[0.4pt]
GILC-DB & 6.3 & 13.5 & 60.5 & 0.1 & 29.6 & \textbf{7.5} & 42.6 & \textbf{4.8} & 22.5 & 14.1 & 23.2 & 0.4 & \underline{41.8} & 16.1 & 67.8 & \underline{4.9} & -0.1 & 448.9 & 2.5 & 735.6 & 48.3 & 618.3 & 11.0 & 1.7\\
\addlinespace[0.4pt]
MDM-VGB & 4.5 & 13.8 & 93.1 & \textbf{0.0} & 39.7 & 10.0 & 62.4 & \underline{5.0} & 25.5 & \textbf{3.0} & 91.8 & \underline{0.0} & 41.0 & \textbf{14.4} & 72.9 & 8.9 & 1.6 & 56.6 & 4.3 & 630.4 & 84.7 & 55.5 & 10.0 & 1.7\\
\addlinespace[0.4pt]
\specialrule{0.55pt}{0.25pt}{0.12pt}
\specialrule{0.25pt}{0pt}{0.80pt}
\rowcolor{oursrow}
\method{} & \textbf{100.0} & \textbf{7.8} & \textbf{100.0} & \textbf{0.0} & \textbf{100.0} & 10.2 & \underline{77.6} & 6.0 & \textbf{100.0} & \underline{7.8} & \textbf{100.0} & \textbf{0.0} & \textbf{100.0} & 20.3 & \textbf{91.6} & 7.2 & \textbf{5.3} & \textbf{1.1} & \underline{21.5} & \underline{16.7} & \underline{99.8} & \textbf{1.1} & \textbf{28.2} & \textbf{1.4}\\
\bottomrule
\end{tabular}
}
\end{table}

\section{Experiments}
\label{sec:experiments}

\subsection{Main Results}
\label{sec:endpoint-results}

We evaluate \method{} with frozen Dream and LLaDA backbones on symbolic and
language tasks, D3LM on DNA, and EvoDiff on protein design, covering exact
constraints and learned objectives.
Dyck repair and Sudoku measure validity alongside edit distance or constraint
violations. CommonGen~\citep{lin-etal-2020-commongen} and
RealToxicityPrompts (RTP)~\citep{gehman-etal-2020-realtoxicityprompts} assess
lexical coverage and non-toxicity alongside perplexity (PPL).
DeepSTARR~\citep{dealmeida2022deepstarr}, APARENT~\citep{bogard2019aparent}, and
Malinois~\citep{gosai2024malinois} use predictor-defined DNA objectives, with
native leave-one-out pseudo-perplexity (LOO) and diversity as diagnostics.
Protein-1YCR follows EvoDiff's motif-scaffolding setup~\citep{alamdari2023evodiff}
and evaluates structural recovery using OmegaFold~\citep{wu2022omegafold}.
Table~\ref{tab:main-results} reports full-budget results;
Fig.~\ref{fig:budget-primary} examines whether target attainment is preserved
with fewer denoising updates. 
Details are 
in Secs.~\ref{app:data-evaluation}--\ref{app:baselines}.

\paragraph{Symbolic tasks.}
\label{sec:symbolic-results}
On Dyck, \method{} and DG-TAG both reach full validity, but \method{} requires
fewer edits. Whereas DG-TAG rescores a bounded set of successors, \method{}'s
min-plus pass first retains the minimum-substitution feasible set within the
declared support. Denoiser evidence and carrier weights then select among
equally optimal repairs, without trading additional edits for a higher
generation score. The Sudoku budget sweep tests a different issue: individually
feasible digits can still be mutually incompatible.
A joint reconstruction supplies a common feasible completion for all of these
choices. 
CDM also maintains full validity across
budgets when its particles operate within exact feasible-completion support,
despite its low validity on Dyck. Thus, the symbolic comparison concerns both
the construction of the feasible set and how parallel choices are coupled
within it.

\paragraph{Language tasks.}
\label{sec:language-results}
On CommonGen, \method{} achieves full coverage across both backbones even
under limited budgets in Fig.~\ref{fig:budget-primary}. 
The compiled
terminal condition excludes completions that omit a requested concept,
while the objective state tracks concepts already covered, so the
remaining completion calculation coordinates the placement of those still
missing. 
Among accepted realizations, word order and surrounding
text remain weighted by the denoiser and carrier. 
On the RTP task, \method{} preserves good PPL with high non-toxic rates over
different budgets and backbones, showing the benefits of joint inference coupled with sequence-level objectives.

\paragraph{Biological sequence design.}
\label{sec:biology-results}
On biological sequences, \method{} improves control metrics with low trade-off 
of sequence quality.
On Protein-1YCR, higher structural success is
accompanied by lower mean motif RMSD. For DNA, low native LOO can coexist
with severe concentration of the outputs. 
K562 combines high target attainment
with few distinct sequences, while DeepSTARR collapses entirely to a single
all-T sequence despite favorable activity and LOO scores. 
The diversity diagnostic exposes the
concentration that neither of these two averages measures.

\begin{figure}[!thbp]
  \centering
  \includegraphics[width=\textwidth]{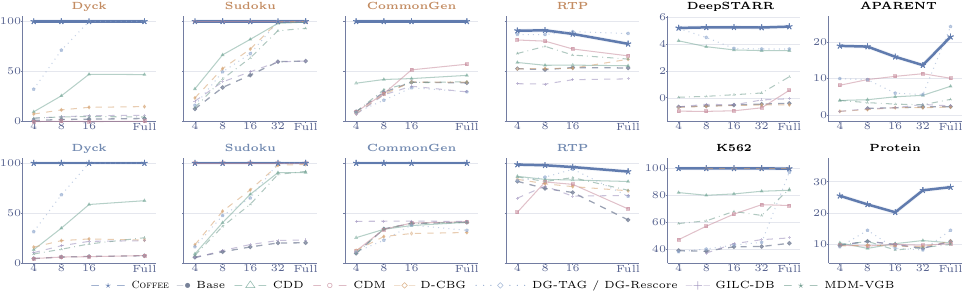}
  \vspace{-0.55em}
  \caption{
\textbf{Control across denoising budgets $b$ (higher is better).}
Warm \textcolor{DreamHeader}{Dream} and blue \textcolor{LLaDAHeader}{LLaDA}
titles distinguish the language backbones; black titles denote biological
tasks. Hollow markers indicate archival or qualified settings documented in
Sec.~\ref{app:complete-results}.
The results show that \method{} achieves the highest control effects over almost all
the benchmarks, even with only a few diffusion steps.
}
  \label{fig:budget-primary}
\end{figure}

\subsection{Mechanistic Analysis}

In the HepG2 rollouts, objective weighting changes the evolution of complete
clean proposals, yielding more favored proposals at the final displayed
update (Fig.~\ref{fig:guidance-mechanism}A). The intervention acts through guided
reconstruction.
After the host commits tokens, the next reconstruction uses refreshed
denoiser evidence and respects those observations. Guidance therefore
repeatedly reweights completions compatible with the evolving partial sequence.
Eq.~\ref{eq:forward-sampling} combines token-level transition factors with
an aggregate weight over the remaining completions. The continuation ratio
in Eq.~\ref{eq:guided-continuation-expectation} averages downstream objective
weight under the neutral model for each current choice. This allows future
objective factors to affect the current token conditional.
The separate
K562 diagnostic exposes this dependence even when downstream token identities
are already fixed: their objective contributions still depend on the state
reached by the current choice. Clamping their evidence does not make those
objective factors common constants that cancel between current candidates.
These state-dependent corrections are substantial at a small subset of the
queried positions (Fig.~\ref{fig:guidance-mechanism}B-C).
Meanwhile, to test completion weighting and joint sampling in complete generations, we
run separate controls over CommonGen. Full lowers PPL
relative to feasibility-only guidance, while independent guided marginals
reduce lexical coverage from 100\% to 83.98\%, showing why joint sampling
matters (Sec.~\ref{app:commongen-components}).
We have also evaluated the empirical inference speed with the same configuration in Table~\ref{tab:main-results}, and shown that
without much efforts in designing graphical model specific operators,
\method{}'s inference speed surpasses most of the baselines while obtaining high control as well as generation quality, indicating the benefits of our plug-and-play scheme. Details are in Sec.~\ref{app:full-generation-cost}.

\begin{figure}[!t]
  \centering
  \includegraphics[width=\textwidth]{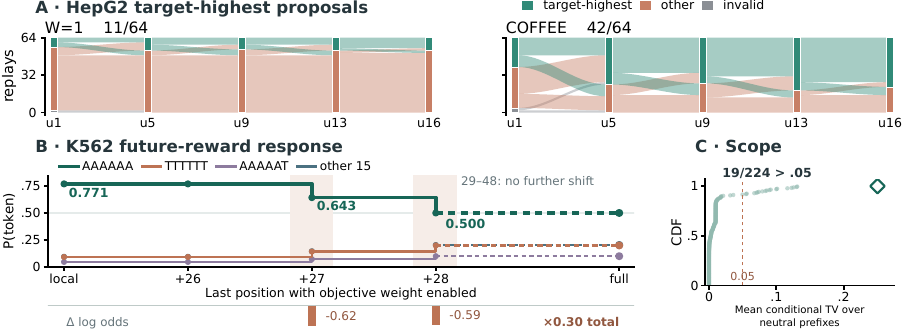}
  \caption{
\textbf{Proposal trajectories and token-level guidance.}
(A) Evaluator-category transitions of complete clean proposals from 64 same-ID
HepG2 rollouts per arm, with the objective disabled ($W_C\equiv1$) or enabled
and the identity carrier fixed. Ribbon widths count trajectories, not
probability mass.
(B) For a fixed K562 query and prefix at position 25, curves show token
probabilities as future objective factors are enabled cumulatively;
the state graph and all other factors remain fixed.
Local omits future weights; Full includes all. Weights at two already observed
downstream tokens reduce the probability of \texttt{AAAAAA}.
Bars show incremental log odds of \texttt{AAAAAA} against the other candidates.
The horizontal axis is objective position, not diffusion time.
(C) CDF of Full-Local total variation over 224 masked positions in nine queries,
with conditional TV averaged over prefixes under neutral occupancy.
The diamond marks the maximum-TV position selected for (B).
Protocols are in Sec.~\ref{app:guidance-diagnostics}.
  }
  \label{fig:guidance-mechanism}
\end{figure}

\subsection{Discussion and Related Work\protect\footnotemark}
\label{sec:discussion}
\footnotetext{A more detailed literature review is in Sec.~\ref{sec:related}.}

The fixed-query analysis in Figure~\ref{fig:guidance-mechanism} shows how
future objective factors can affect a current choice without changing the
denoiser evidence,
which connects \method{} to the probabilistic-control view of
generation~\citep{levine2018inference}, where preferences over complete
outcomes inform intermediate decisions.
Future discriminators, soft-value decoding, and twisted sequential
Monte Carlo obtain such signals through prediction or
sampling~\citep{yang2021fudge,li2025svdd,zhao2024twisted},
while GeLaTo, Ctrl-G, and TRACE use tractable completion
models~\citep{zhang2023gelato,zhang2024ctrlg,weng2025trace}.
\method{} combines this completion-based view with the dependence
modeling of CoDD~\citep{li2026codd}, and within each diffusion step, it
aggregates objective contributions over joint reconstructions rather
than treating unresolved choices independently to model cross-token dependency.

The compiled representation also supports requirements that are more
naturally expressed as inference queries than as a single predictive score, as discussed in Sec.~\ref{sec:extension}.
Our Dyck repair construction first identifies all minimum-edit valid
paths within the declared support, then samples among them using
denoiser evidence and carrier weights.
This combines optimization with probabilistic conditioning, drawing on
weighted-automata algorithms~\citep{mohri2009weighted},
connecting to neuro-symbolic structured prediction, where neural
evidence is combined with explicit
constraints~\citep{ahmed2022spl,vanKrieken2025neurosymbolic}.
Compatible objective states could similarly combine hard requirements
and soft preferences in one reconstruction query.
The benefit is therefore not only a change in the scoring function,
but access to additional more complex algorithmic computations during generation.

These capabilities depend on finding representations that remain both
useful and tractable when combined.
Efficient evaluation of the model and scorer separately does not ensure
an efficient expectation query~\citep{khosravi2019expected}, and research in
knowledge compilation and circuit operations makes representation size
and structural compatibility central to this
question~\citep{darwiche2002knowledge,vergari2021atlas}.
Our selected token and phrase features omit interactions outside the
scorer class, while richer representations can make the product state
too large.
Proposition~\ref{prop:compilation-approximation} bounds objective
approximation for a fixed carrier and support,
but more theories need to be developed to
extend the expressiveness of such representation, and to
analyze the diffusion errors alongside the generation process.
Circuit restructuring and structured-sparse parameterizations offer
related directions for compatible composition and hardware-efficient
inference~\citep{zhang2025restructuring,zhang2025monarch}.

\section{Conclusion}
\label{sec:conclusion}

We introduced \method{}, a plug-and-play framework that turns
sequence-level objectives into joint reconstruction queries for
discrete diffusion.
A target-free carrier and compiled objective allow the sampler to
coordinate unresolved tokens and account for their possible
completions without explicit enumeration.
Under the stated representation and support conditions, fixed-step
inference is exact for the compiled model and leaves the pretrained
generator unchanged.
The same interface supports hard conditioning, finite-state soft
preferences, and optimization followed by conditional sampling.
Experiments show full satisfaction on the evaluated symbolic and
lexical tasks and success
in reward-guided biological sequence generation.
More broadly, the representation of a guidance objective determines
not only which sequences it rewards, but also which queries can be
answered during generation.
Making objectives available to conditioning, marginalization, and
optimization enables control over sets of possible completions,
rather than requiring those completions to be generated and scored
individually, revealing broader potentials for neural-symbolic methods.

\section*{AI Use Statement}

Generative AI tools were used for literature search and synthesis, feedback on
research methodology and experiment design, code implementation and debugging,
analysis and interpretation of results, manuscript organization and revision,
and preparation of LaTeX tables and figures. They were not used to fabricate or
alter experimental measurements, to choose held-out results after observing
their outcomes, or to make final scientific decisions. The authors checked
AI-assisted citations against the original sources, reviewed and tested
AI-assisted code, and reconciled reported values and claims with the frozen
experimental artifacts and evaluator outputs. The authors made all final
scientific and writing decisions and take responsibility for the complete
contents of this preprint.

\section*{Ethics Statement}

This computational study recruited no human participants, collected no new
personal data, and performed no wet-lab biological experiments. The language
evaluation includes potentially harmful text from an established toxicity
benchmark, which is used only for computational evaluation. The proposed
framework is dual-use because it can steer language or biological sequences
toward any objective that admits a compatible representation. Although the
evaluated objectives concern non-toxic language and established biological
design benchmarks, predictor scores do not establish biological function,
safety, or clinical utility. We therefore restrict our claims to the declared
evaluators and supports, report failure modes such as mode collapse, and make
no claim of experimentally validated biological efficacy.

\section*{Reproducibility Statement}
The \hyperref[app:guide]{appendix guide} summarizes the reproducibility path.
The method section and Appendices~\hyperref[app:formal-setup]{A} to
\hyperref[app:reward-training]{C} specify the fixed-step law, exactness and
complexity boundaries, and construction of the carrier and objectives.
Appendices~\hyperref[app:data-evaluation]{D} to
\hyperref[app:experimental-protocol]{E} document the data and evaluator contracts,
baseline adaptations, development-only selection rules, and frozen configurations.
Tab.~\ref{tab:main-results} and Fig.~\ref{fig:budget-primary} contain the
claim-facing results, while Appendix~\hyperref[app:complete-results]{F}
documents their qualifications. Appendices~\hyperref[app:algorithm]{G} to
\hyperref[app:reproducibility]{I} provide the canonical algorithms, failure
semantics, diagnostic protocols, and evidence provenance used to audit the
reported claims.

\bibliographystyle{iclr2027_conference}
\bibliography{references}

\appendix
\section*{Guide to the Appendix}
\label{app:guide}

This appendix follows the objects and claims introduced in the main paper.
Secs.~\hyperref[app:formal-setup]{A} to \hyperref[app:reward-training]{C}
give the formal derivations, computational boundary, and construction of the
carrier and objectives. Secs.~\hyperref[app:data-evaluation]{D} to
\hyperref[app:complete-results]{F} specify the evaluation contracts, comparison
rules, frozen configurations, and qualifications needed for the main results. The final three
sections provide canonical algorithms, process diagnostics, and evidence
provenance. Current benchmark endpoints are kept separate from DEV or
TRAIN-derived diagnostics and from historical cohorts.

\paragraph{Reader conventions.}
TRAIN, DEV, and TEST denote fitting, configuration selection, and final
evaluation roles. A \emph{frozen configuration} has its model, objective,
hyperparameters, prompt or structured input, and evaluator fixed before final
evaluation. The \emph{host} is the unchanged pretrained diffusion backbone and
its original update rule. A \emph{port} is our implementation of a published
algorithmic idea on that host and does not imply author-code identity. NFE denotes
the realized number of denoiser row-forward evaluations. Full denotes the
task-native full budget used in Tab.~\ref{tab:main-results}; it does not imply
equal total compute across methods.

\begin{center}
\small
\renewcommand{\arraystretch}{1.08}
\begin{tabularx}{\linewidth}{@{}p{0.08\linewidth}Xr@{}}
\toprule
Section & Contents & Page \\
\midrule
\hyperref[app:formal-setup]{A} & Additional method details and guarantees & \pageref*{app:formal-setup} \\
\hyperref[app:complexity]{B} & Implementation boundary and computational cost & \pageref*{app:complexity} \\
\hyperref[app:reward-training]{C} & Carrier and objective construction & \pageref*{app:reward-training} \\
\hyperref[app:data-evaluation]{D} & Experimental setup and evaluation & \pageref*{app:data-evaluation} \\
\hyperref[app:baselines]{E} & Baselines and hyperparameter selection & \pageref*{app:baselines} \\
\hyperref[app:complete-results]{F} & Main-claim support and result qualifications & \pageref*{app:complete-results} \\
\hyperref[app:algorithm]{G} & Canonical algorithms and failure semantics & \pageref*{app:algorithm} \\
\hyperref[app:qualitative]{H} & Guidance dynamics and diagnostic experiments & \pageref*{app:qualitative} \\
\hyperref[app:reproducibility]{I} & Reproducibility and evidence provenance & \pageref*{app:reproducibility} \\
\bottomrule
\end{tabularx}
\end{center}

For a direct route from the main paper, the fixed-step law and exactness lead
to Secs.~\hyperref[app:formal-setup]{A} to \hyperref[app:complexity]{B}; the
carrier and fitted objectives lead to Sec.~\hyperref[app:reward-training]{C};
the benchmark table and budget curves lead to
Secs.~\hyperref[app:data-evaluation]{D} to \hyperref[app:complete-results]{F};
and Fig.~\ref{fig:guidance-mechanism} leads to
Sec.~\hyperref[app:qualitative]{H}.

\section{Additional Method Details and Guarantees}
\label{app:formal-setup}

\subsection{Fixed-step interface and HMM realization}

This appendix states precisely what is normalized, what is compiled, and where
approximations enter.  Fix one denoising update and one block.  Let
$x=(x_1,\ldots,x_L)$ be the block tokens and let
$\mathcal{S}=\prod_i\mathcal{S}_i$ be the candidate support fixed before the
dynamic program. At the interface level, $e_{t,i}$ is frozen denoiser evidence,
$K_i$ is a nonnegative target-free carrier factor, and
$(\psi_{C,i},\eta_C)$ is the compiled objective. These are the same objects as
in Sec.~\ref{sec:backward-dp}; the appendix does not redefine the method as
an HMM.

The target-free hidden Markov model (HMM) used by the structured-carrier implementation is one
realization of $K_i$. It has latent path
$h=(h_1,\ldots,h_L)$, initial distribution $\pi$, transition matrix $A$, and
token emissions $E$.  With frozen denoiser log evidence $\ell_i(v)$, evidence
temperature $\tau_{\rm e}$, and carrier-emission temperature $\tau_{\rm m}$,
the retained base score is
\begin{align}
\log q_0^{\mathcal{S}}(x,h)
\doteq{}&
\sum_{i=1}^{L}\frac{\ell_i(x_i)}{\tau_{\rm e}}
+\log \pi(h_1)
+\sum_{i=2}^{L}\log A(h_i\mid h_{i-1}) \nonumber\\
&+\sum_{i=1}^{L}\frac{\log E(x_i\mid h_i)}{\tau_{\rm m}},
\qquad x\in\mathcal{S},
\label{eq:base-score}
\end{align}
where $\doteq$ denotes equality up to a path-independent normalizer.  For a
masked position, $\ell_i(v)=\log p_\theta(x_0^i=v\mid x_t)$; the effective
unary factor is proportional to $\exp\{\ell_i(v)/\tau_{\rm e}\}$, reducing to
Eq.~\ref{eq:evidence} at $\tau_{\rm e}=1$. Observed tokens remain clamped.
For a fixed prefix, $\pi$ is replaced by the predictive distribution for the first
block state: filter the prefix to obtain its final-state posterior, then apply
one carrier transition before scoring the first block emission. The neutral
one-state identity carrier is represented by a unit $K_i$ factor.

Let $a_i$ be the reward state and $s_i=(h_i,a_i)$ the product state.  For
$i>1$, an edge $e_i=(s_{i-1},v,s_i)$ has log score
\begin{align}
g_i(e_i)
={}&\frac{\ell_i(v)}{\tau_{\rm e}}
+\log A(h_i\mid h_{i-1})
+\frac{\log E(v\mid h_i)}{\tau_{\rm m}}
+\lambda r_i(a_{i-1},v),
\label{eq:edge-score}
\end{align}
subject to $v\in\mathcal{S}_i$ and
$a_i=\delta(a_{i-1},v)$.  The first edge uses this predictive carrier distribution (or $\pi$ without
a prefix) in place of a transition from the dummy state $h_0$.

Let
$\mathcal{G}=(\mathcal{N},\mathcal{E})$ be the resulting layered acyclic graph.
A vertex at level $i$ is a retained product state $s_i=(h_i,a_i)$; an edge
$e_i=(s_{i-1},x_i,s_i)$ is present only when $x_i$ is in the declared token
support, the carrier transition is allowed, and
$a_i=\delta(a_{i-1},x_i)$.  Define
\begin{equation}
  W(\xi)=\exp\!\left\{\sum_{i=1}^{L}g_i(e_i)
  +\log\eta_C(a_L)\right\}
  \label{eq:appendix-path-weight}
\end{equation}
for a complete path $\xi=(e_1,\ldots,e_L)$, using $\log 0=-\infty$.
The terminal factor $\eta_C$ is the one in
Eq.~\ref{eq:objective-interface}; it is not multiplied by the soft
reward strength.  The path partition is
$Z_{\mathcal{G}}=\sum_{\xi\in\mathcal{G}}W(\xi)$.  The method samples
$W(\xi)/Z_{\mathcal{G}}$ when $0<Z_{\mathcal{G}}<\infty$ and the ancestral
temperature is one.
The special cases in Tab.~\ref{tab:special-cases} share this construction.

\begin{table}[h]
\centering
\small
\setlength{\tabcolsep}{4pt}
\caption{Special cases of the same graph construction.  The identity carrier
is an ablation of the full method; the identity reward controller has one
state.}
\label{tab:special-cases}
\begin{tabular}{p{0.25\linewidth}p{0.25\linewidth}p{0.35\linewidth}}
\toprule
Instance & Product state & Reward factor \\
\midrule
Diffusion-only & identity & $1$ \\
Carrier-only & $h_i$ & $1$ \\
Learned guidance & $(h_i,a_i)$ & $\exp\{\lambda R_\phi(x)\}$ \\
Hard constraint & $(h_i,a_i)$ & $\mathbf{1}[a_L\in\mathcal{A}_{\rm accept}]$ \\
Strict repair & $(h_i,\text{bounded stack}_i)$ &
$\mathbf{1}[x\in\mathcal{L},\ d_H(x,y)=d^\star(y)]$ \\
\bottomrule
\end{tabular}
\end{table}

\subsection{Exact compilation of count energies}

We suppress the fixed condition index $C$ here and allow position-dependent
token coefficients $u_i$. Shared token
coefficients are the special case $u_i(u)=u_u$.

Let $\mathcal{P}=\{p_j\}_{j=1}^{J}$ be the selected token patterns and let the
Aho--Corasick output set $\mathcal{O}(a_{i-1},x_i)$ contain every pattern that
ends after consuming $x_i$~\citep{aho1975efficient}.  The compiler assigns
\begin{equation}
 r_i(a_{i-1},x_i)
 =u_i(x_i)+\sum_{j:p_j\in\mathcal{O}(a_{i-1},x_i)}w_j.
 \label{eq:compiled-edge-reward}
\end{equation}

\begin{lemma}[Count-energy compilation]
\label{lem:count-energy-compilation}
For every token sequence $x$, including sequences with overlapping pattern
occurrences,
\[
\sum_{i=1}^{L}r_i(a_{i-1},x_i)
=\sum_{i=1}^{L}u_i(x_i)+\sum_{j=1}^{J}w_jc_j(x)=R_\phi(x)-b.
\]
\end{lemma}

\begin{proof}
The token term contributes once at every position.  Aho--Corasick reports a
pattern exactly at each position at which one occurrence ends; output links
retain simultaneous suffix matches, so overlapping and nested occurrences are
not discarded.  Summing the reported $w_j$ terms therefore contributes
$w_j$ exactly $c_j(x)$ times.  The fitted intercept is absent because it is
constant across all paths in the same normalized kernel.
\end{proof}

\paragraph{Automatic weighted-automaton construction.}
The compiler inserts every selected token pattern into a trie.  Breadth-first
failure links map a missing transition to the longest suffix that is also a
trie prefix.  The reward stored at a state is its own terminal weight plus the
reward at its failure state.  Consequently one transition emits the sum of
\emph{all} selected patterns ending at that position, including nested and
overlapping matches.  For the candidate tokens admitted by a denoising block,
the compiler caches $\delta(a,v)$ and the phrase-match reward for every
reachable state $a$ and token $v$; the position-dependent token term is added
when scoring the edge at position $i$.

\paragraph{Record-conditioned lexical DFA.}
CommonGen uses the same transition interface but no learned weights.  For each
record, the adapter enumerates the accepted standalone forms of every concept:
the base form, its regular plural (\texttt{s}, \texttt{es}, or
\texttt{ies}), and its possessive form.  These finite token strings are
compiled into a lexical prefix machine.  A DFA state stores the active form
prefixes together with a bit mask of concepts already completed at word
boundaries; a forbidden concept, when present, enters a reject sink.  The
terminal score is zero only when every required bit is set and no reject state
has been reached, and is $-\infty$ otherwise.  We enumerate only states
reachable under the declared per-position token support.  If that support
cannot reach acceptance, the adapter fails closed rather than softening the
constraint.

\subsection{Exact compilation and oracle approximation}
\label{app:compilation-approximation}

Exact compilation means that the local transition and terminal weights multiply
to a positive sequence-independent constant times the fitted sequence weight.
For the count score, Lem.~\ref{lem:count-energy-compilation} gives
\[
\prod_{i=1}^{L}\exp\{\lambda r_{C,i}(a_{i-1},x_0^i)\}
=\exp\{-\lambda b_C\}\exp\{\lambda R_{\phi,C}(x_0)\}.
\]
The first factor cancels from the guided distribution. For a hard recognizer,
the compiled weight equals the acceptance indicator, including its zero set.
Thus compilation implements the chosen objective without changing its guided
law. It does not guarantee that a fitted objective matches an unrestricted oracle.

\begin{proof}[Proof of Prop.~\ref{prop:compilation-approximation}]
All distributions below use the same fixed $x_t$, base distribution, and support.
Let $d(x_0)=\log W_C(x_0)-\log W_C^{\mathrm{oracle}}(x_0)$ and define
\[
A(s)=\log\mathbb E_{q_C^{\mathrm{oracle}}}[e^{s d(x_0)}],
\qquad
q_s(x_0)=q_C^{\mathrm{oracle}}(x_0)e^{s d(x_0)-A(s)}.
\]
Then $q_1=q_C$, $A'(s)=\mathbb E_{q_s}[d]$, and
$A''(s)=\operatorname{Var}_{q_s}(d)$. Since $|d-c|\leq\epsilon$,
$A''(s)\leq\mathbb E_{q_s}[(d-c)^2]\leq\epsilon^2$. Consequently,
\[
D_{\mathrm{KL}}(q_C^{\mathrm{oracle}}\Vert q_C)
=A(1)-A(0)-A'(0)
=\int_0^1(1-s)A''(s)\,ds\leq\epsilon^2/2.
\]
Replacing the integrand by $sA''(s)$ gives the same bound in the reverse
direction. Exact compilation leaves $q_1$ unchanged by the constant-cancellation
argument above.
\end{proof}

\paragraph{Function-class gap and actual fitting error.}
On this finite support, let $\mathcal F_m$ be a class of log-weights that admit
exact compilation under representation budget $m$. Define
\[
\delta_m=\inf_{f\in\mathcal F_m,\,c\in\mathbb R}
\|f-\log W_C^{\mathrm{oracle}}-c\|_\infty.
\]
For the actual fitted log-weight $\log W_{C,m}\in\mathcal F_m$, let
$\epsilon_m=\inf_c\|\log W_{C,m}-\log W_C^{\mathrm{oracle}}-c\|_\infty$
and $\rho_m=\epsilon_m-\delta_m\geq0$. The proposition applies with
$\epsilon_m=\delta_m+\rho_m$, not with the class infimum alone.
Nested classes make $\delta_m$ nonincreasing, but do not imply that it tends to
zero. If both $\delta_m$ and $\rho_m$ tend to zero, the fixed-step KL gap does
too. Without attainment, $\delta_m=0$ means membership in the closure modulo
constants, not necessarily an exact finite representation.

These are conditional approximation statements, not convergence guarantees
for the current training procedure. Prediction loss on training data does not
establish uniform log-weight error on generated sequences. For positive
Boltzmann weights, the comparison is between log-weights, so guidance strength
is already included. If weights can be zero, first require matching feasible
support and restrict the comparison to its positive part. The result does not
cover support mismatch, arbitrary pruning, non-unit sampling temperature, or
the full multi-step diffusion distribution.

\paragraph{The three normalizers.}
The main text uses three related but non-interchangeable quantities. On the
same fixed support,
\begin{equation}
Z_0(x_t)=B_0(h_0),\qquad
Z_C(x_t)=\mathbb E_{q_0}[W_C]
=\frac{M_0(h_0,a_{\mathrm{init}})}{B_0(h_0)}.
\label{eq:appendix-normalizer-bridge}
\end{equation}
Thus $B_0$ normalizes the neutral carrier--evidence joint, $M_0$ is the
unnormalized total weight after adding the compiled objective, and $Z_C$ is the
expected objective weight under the already normalized $q_0$. When a scorer
intercept is omitted during compilation, the normalized guided law is
unchanged, but $M_0$ is multiplied by the corresponding path-independent
constant; all partition values in this appendix use the intercept-free
compiled-weight convention.

\paragraph{Meaning of the backward-weight ratio.}
For a base-reachable state with $B_i(h)>0$, start the objective path at
$a_i=a$. The notation $a_i:=a$ initializes the objective automaton for the
continuation; it does not condition the neutral law on a past objective state.
Expanding the expectation in Eq.~\ref{eq:guided-continuation-expectation} gives
\begin{equation}
\begin{aligned}
&\mathbb E_{q_0}\!\left[
\eta_C(a_L)\prod_{j=i+1}^{L}\psi_{C,j}(a_{j-1},x_0^j)
\,\middle|\,h_i=h,x_t;\ a_i:=a\right]\\
&=\sum_{x_0^{i+1:L},h_{i+1:L}}
\frac{\prod_{j=i+1}^{L}e_{t,j}(x_0^j)K_j(h_{j-1},x_0^j,h_j)}{B_i(h)}
\eta_C(a_L)\prod_{j=i+1}^{L}\psi_{C,j}(a_{j-1},x_0^j)\\
&=\frac{M_i(h,a)}{B_i(h)}.
\end{aligned}
\label{eq:remaining-objective-expectation}
\end{equation}
Eq.~\ref{eq:remaining-objective-expectation} follows by dividing the
remaining product weights by their carrier-only normalizer $B_i(h)$ and taking
the expectation of the remaining objective factors, starting at $a_i=a$.
Thus $M_i$ itself is generally not a probability. For indicator constraints,
$M_i/B_i$ is the continuation acceptance probability, with previous violations
recorded in the objective state. For soft objectives it is an expected weight
and may exceed one. This identity uses the unpruned product construction, or
the same fixed per-position token support in numerator and denominator. If
product-state pruning makes the remaining support depend on $a$, the matching
carrier-only normalizer must instead be computed on that same retained graph
and generally depends on both $h$ and $a$.

\paragraph{Guided sampling conditional.}
\label{app:guided-conditional}
Fix $(h,a)$ with $M_{i-1}(h,a)>0$, and let $a'=\delta_{C,i}(a,u)$.
Marginalizing the remaining assignments gives
\begin{equation}
\begin{aligned}
&q_C^{\mathrm{joint}}(x_0^i=u,h_i=h'\mid h_{i-1}=h,a_{i-1}=a,x_t)\\
&\quad=\sum_{x_0^{i+1:L},h_{i+1:L}}
q_C^{\mathrm{joint}}(x_0^i=u,h_i=h',x_0^{i+1:L},h_{i+1:L}
\mid h_{i-1}=h,a_{i-1}=a,x_t)\\
&\quad=\frac{e_{t,i}(u)K_i(h,u,h')\psi_{C,i}(a,u)M_i(h',a')}
{M_{i-1}(h,a)}\\
&\quad=q_0(x_0^i=u,h_i=h'\mid h_{i-1}=h,x_t)\,
\psi_{C,i}(a,u)\frac{M_i(h',a')}{B_i(h')}
\frac{B_{i-1}(h)}{M_{i-1}(h,a)}.
\end{aligned}
\label{eq:guided-conditional-derivation}
\end{equation}
The last equality uses Eq.~\ref{eq:carrier-conditional} on base-supported
choices with $B_i(h')>0$, with the same continuation support in both messages.
Since $B_{i-1}(h)/M_{i-1}(h,a)$ is constant across $(u,h')$, this gives
Eq.~\ref{eq:forward-sampling}. The direct fraction in the second equality
computes the conditional without a separate carrier-only backward pass.

\subsection{Objective-agnostic cache reuse}

For corpus rows $\mathcal{D}$ and mining configuration $\eta$, the candidate
miner is a deterministic map $\mathcal{M}(\mathcal{D},\eta)$ based only on
tokenized document frequency.  Its count matrix is likewise a deterministic
map of the rows and candidate grammar.  Therefore, changing a target label
while holding $\mathcal{D}$ and $\eta$ fixed does not require re-mining or
recounting.  Target changes rerun only score tests, selection, and fitting.
This separation is an algorithmic contract, not merely a cache optimization.

\subsection{Strict minimum-edit repair}
\label{app:repair}

Let $y$ be an observed sequence, $F$ a set of positions that must remain
fixed, $\mathcal{L}$ the valid language, and $\mathcal{S}$ the declared token
support.  We first compute
\begin{equation}
d^\star(y)=\min_{x\in\mathcal{S}\cap\mathcal{L},\;x_F=y_F}d_H(x,y).
\label{eq:repair-distance}
\end{equation}
The repair kernel is then
\begin{equation}
q_{\rm repair}(x,h\mid y)\propto
q_0^{\mathcal S}(x,h)\,
\mathbf{1}[x\in\mathcal{L},\ x_F=y_F]\,
\mathbf{1}[d_H(x,y)=d^\star(y)].
\label{eq:repair-law}
\end{equation}
A forward min-plus pass gives the cheapest prefix cost for every bounded-stack
state, and a backward min-plus pass gives its cheapest valid completion.  An
edge is retained exactly when its prefix cost, local substitution cost, and
suffix cost sum to $d^\star(y)$.  Ordinary sum-product messages then weight
this optimal subgraph by frozen denoiser evidence and the target-free carrier;
Alg.~\ref{alg:strict-repair} samples from Eq.~(\ref{eq:repair-law}).

The present contract does not cover insertions, deletions, an unbounded
pushdown language, or dynamic multi-step repair evidence. The development-only
repair diagnostic is reported with the other diagnostics in
App.~\ref{app:dyck-repair-diagnostic}.

\subsection{Exact backward inference and ancestral replay}
\label{app:exact-replay}

\begin{proposition}[Exactness of the compiled fixed-step query]
\label{prop:compiled-fixed-step}
Fix $x_t$, the evidence, and finite carrier and objective state spaces. Assume
all factors in Eq.~\ref{eq:product-law} are finite and nonnegative, with
$0<Z=M_0(h_0,a_{\mathrm{init}})<\infty$. In exact arithmetic,
Eq.~\ref{eq:product-recurrence} and the unit-temperature conditionals in
Eq.~\ref{eq:forward-sampling} sample the normalized compiled joint.
\end{proposition}
\begin{proof}
The terminal message is the remaining weight $\eta_C(a_L)$. Backward
induction shows that each preceding message sums all compatible remaining
factor products. Thus the denominator of Eq.~\ref{eq:forward-sampling}
is the sum of its numerators. Only positive-mass successors can be sampled,
so each subsequent conditional is defined. For a complete sampled path,
\begin{align*}
\prod_{i=1}^{L}
\frac{e_{t,i}(x_0^i)K_i(h_{i-1},x_0^i,h_i)
\psi_{C,i}(a_{i-1},x_0^i)M_i(h_i,a_i)}{M_{i-1}(h_{i-1},a_{i-1})}
\\[0.4em]
{}=\frac{\eta_C(a_L)}{Z}
\prod_{i=1}^{L}e_{t,i}(x_0^i)K_i(h_{i-1},x_0^i,h_i)
\psi_{C,i}(a_{i-1},x_0^i).
\end{align*}
All intermediate messages cancel. Summing out the carrier path recovers
Eq.~\ref{eq:guided-law}. The argument also applies to a fixed restricted
graph when the same edges and factors are used for inference and sampling.
\end{proof}

\paragraph{Log-domain form on a retained graph.}

For a state $s_i$ define the suffix mass
\begin{equation}
 \beta_i(s_i)=
 \sum_{\xi_{i+1:L}\mid s_i}
 \exp\!\left\{\sum_{j=i+1}^{L}g_j(e_j)
 +\log\eta_C(a_L)\right\}.
 \label{eq:beta-definition}
\end{equation}
At the terminal level,
$\beta_L(h_L,a_L)=\exp\{\log\eta_C(a_L)\}$.  The ordinary sum-product
recursion is
\begin{equation}
 \beta_{i-1}(s_{i-1})=
 \sum_{e_i:s_{i-1}\to s_i}\exp\{g_i(e_i)\}\beta_i(s_i).
 \label{eq:beta-recursion-prob}
\end{equation}
Eq.~(\ref{eq:product-recurrence}) is the same probability-space recursion
with the denoiser, carrier, and objective factors written explicitly.

\begin{corollary}[Exact replay on a retained graph]
\label{cor:retained-graph-replay}
Under Prop.~\ref{prop:compiled-fixed-step}, fix the candidate tokens and
the retained product-state graph. Sampling each edge according to
\begin{equation}
 p(e_i\mid s_{i-1})
 =\frac{\exp\{g_i(e_i)\}\beta_i(s_i)}
 {\beta_{i-1}(s_{i-1})}
 \label{eq:appendix-edge-conditional}
\end{equation}
produces an exact sample from $W(\xi)/Z_{\mathcal{G}}$ when
$0<Z_{\mathcal G}<\infty$ and $\tau_{\rm a}=1$. This is the same telescoping
argument as the proposition, restricted to the fixed retained graph rather
than a second exactness claim.
\end{corollary}

\begin{corollary}[Hard-DFA conditioning]
With $\eta_C(a_L)=\mathbf 1[a_L\in\mathcal A_{\mathrm{accept}}]$
and no soft reward factors, exact replay samples the retained base distribution
conditioned on DFA acceptance. This indicator is independent of $\lambda$,
so hard constraints remain active even when soft guidance has zero strength.  If no
accepting path exists, $Z_{\mathcal{G}}=0$ and the correct behavior is a
fail-closed infeasibility record.
\end{corollary}

\paragraph{Why non-unit ancestral temperature is different.}
Dividing each prefix-conditional logit by $\tau_{\rm a}\neq1$ defines a valid
tempered policy.  Because the backward message in one conditional also changes
the state reached by later conditionals, the product of these locally tempered
conditionals is not generally proportional to a single global power
$W(\xi)^{1/\tau_{\rm a}}$.  We therefore report it as a policy choice rather
than as an exact sample from Eq.~(\ref{eq:guided-law}).

For example, consider three complete paths with weights $(1,1,4)$, where the
first two share their first branch. At $\tau_{\rm a}=2$, locally tempering the
prefix conditionals gives probabilities approximately
$(0.2071,0.2071,0.5858)$, whereas globally normalizing the square-root path
weights gives $(0.25,0.25,0.50)$. Positive local temperature can preserve hard
zero support, but it does not preserve the unit-temperature sampling law.

\section{Implementation and Computational Cost}
\label{app:complexity}

\subsection{Graph-level time and memory}

Let $V_i\subseteq\mathcal N$ be the retained states at level $i$ and $E_i$ the retained edges from
level $i-1$ to $i$.  Write $|V|=\sum_i|V_i|$ and
$|E|=\sum_i|E_i|$.

\begin{theorem}[Finite-graph complexity]
The backward partition computation takes $\mathcal{O}(|E|)$ arithmetic
operations.  Retaining all messages for replay uses $\mathcal{O}(|V|)$
message memory; partition-only inference can stream levels using
$\mathcal{O}(\max_i|V_i|)$ message memory.  One replay is no more expensive
than enumerating the outgoing edges of its $L$ visited states and is dominated
by the preceding backward pass.
\end{theorem}

\begin{proof}
Every backward message sums each outgoing retained edge once, so the total work
is linear in the edge list.  A stored scalar message per retained vertex gives
the replay bound; if no replay is required, the acyclic recurrence needs only
two adjacent levels.  Replay visits exactly one state per level.
\end{proof}

For candidate width $K$, $H$ carrier states, and $M$ retained reward states,
a dense Cartesian graph has $\mathcal{O}(LMH)$ vertices and at most
$\mathcal{O}(LKMH^2)$ edges.  The deterministic automaton transition fixes the
next reward state; it does not introduce another factor of $M$.  With the
identity carrier, $H=1$ and the time bound is $\mathcal{O}(LKM)$.  In the
implemented batched version, batch size multiplies graph work while enabling
GPU parallelism; padding and packing determine the constant factor.

These bounds begin after the finite graph is specified. They do not make graph
construction polynomial in every compact problem description: bounded stacks,
concept bit masks, multiple controller products, and explicit Sudoku support
can make the expanded state space large. Reported computational cost therefore
separates support construction or pre-solving, retained edge storage, message
workspace, backward arithmetic, and host-model calls. The
$\mathcal O(|E|)$ and $\mathcal O(|V|)$ statements apply only to message
passing on the materialized graph.

\subsection{Measured Full-generation time}
\label{app:full-generation-cost}

The graph-work bound alone does not determine wall time. Wide product states
can require specialized GPU reductions and sampling operators to avoid many
small launches and host--device synchronizations. We implemented targeted
fast paths, including batched message passing and fused ancestral sampling,
but did not exhaustively engineer a separate kernel for every task and
hardware setting. Tab.~\ref{tab:full-generation-cost} measures the resulting
implementations; it is not a lower bound on the method's attainable latency.

The primary measurements use NVIDIA RTX 4090 GPUs; measurements from a
separate cloud GPU environment are marked in the table. Each row uses the frozen Full schedule: $b=32$ for Dyck,
CommonGen, and RTP, $b=64$ for Sudoku, $b=48$ for DNA, and one editable site
per update for Protein-1YCR. Models and task-specific configurations follow
Tab.~\ref{tab:main-results} and the DEV selections in
Tab.~\ref{tab:appendix-selection-contracts}. Each reported window generates
64 final outputs with the method's native internal particle, classifier, or
search work. Static model loading is separate. Timed steady-state windows
follow a real warmup; labeled first-request windows include Graph capture or
just-in-time kernel setup. CUDA is synchronized before and after complete
generation from prepared inputs to CPU-available tokens. Request-specific
controller construction, dynamic programming,
ancestral sampling, transfers, and online model calls are included; independent
evaluation is excluded. External output batches and precision follow each
frozen implementation. Ratios are interpreted only within a matched GPU environment, precision,
external batch, and timer boundary; input or cohort differences do not
by themselves establish an algorithmic speed difference.

\begin{table}[H]
\centering
\scriptsize
\setlength{\tabcolsep}{2.5pt}
\renewcommand{\arraystretch}{1.07}
\caption{Complete online generation time in seconds for 64 final outputs at
each task's frozen Full budget. Base is a compute reference; bold identifies
the fastest measured non-Base baseline within a GPU environment. BoN-4 is excluded.
Static model loading, warmup, disk output, and independent evaluation are
excluded. Per-request controller construction, inference, sampling, transfers,
and actual extra model calls are included; early stopping is retained.
$^{\dagger}$Measurements from a separate cloud GPU environment are shown for
context and must not be divided by primary RTX 4090 measurements in the same
cell. $^{\ddagger}$This D-CBG setting
failed its predeclared DEV PPL cap; its measured time remains visible.
The RTP/LLaDA DG-TAG and GILC ranges reverse order across repeats and use
different native external batches. The LLaDA/Sudoku MDM range includes a
different-card repeat; the same-GPU comparison with \method{} uses 72.364.
$^{\rm w}$Reused static Graph, with complete generation timed;
$^{\rm f}$first request including Graph capture. Times with different native
batches, cohorts, or software stacks are descriptive unless the measurement
conditions are matched.}\label{tab:full-generation-cost}
\begin{tabular}{@{}lrrrrrrrr@{}}
\toprule
Task / backbone & Base & CDD & CDM & D-CBG & DG-TAG & GILC-DB & MDM-VGB & \method{} \\
\midrule
Dyck / Dream & 6.249 & 96.966 & \textbf{51.076} & 82.366 & 121.766 & 443.709 & 406.640 & 13.940 \\
Dyck / LLaDA & 6.237 & 107.607 & \textbf{76.587} & 107.587 & 136.590 & 469.933 & 389.544 & 12.086 \\
Sudoku / Dream & 50.592 & 105.776 & 67.333 & 95.105 & 83.739 & 111.114 & \textbf{53.027} & 56.295 \\
Sudoku / LLaDA & 69.252 & 110.038 & 81.719 & 110.513 & 101.191 & 112.844 & \textbf{71.603--72.364} & 75.136 \\
CommonGen / Dream & 20.121 & 38.690 & 111.330 & \textbf{20.272}$^{\ddagger}$ & 80.754 & 30.454--30.649 & 624.670 & 109.432 \\
CommonGen / LLaDA & 15.428 & 68.535 & 152.613 & 79.558 & 76.392 & \textbf{62.749} & 590.824 & 65.442 \\
RTP / Dream & 21.041 & 116.940 & 106.992 & 719.688 & 131.728 & \textbf{21.094} & 590.266 & 76.466 \\
RTP / LLaDA & 20.596 & 129.462 & 1098.473 & 102.751 & \textbf{93.482--93.703} & \textbf{92.419--96.943} & 821.833 & 154.703 \\
\midrule
DeepSTARR / D3LM & 0.917$^{\dagger}$/1.210 & \textbf{0.968$^{\dagger}$/1.254} & 6.996$^{\dagger}$ & 146.770$^{\dagger}$ & 31.984$^{\dagger}$ & 26.463$^{\dagger}$ & 17.800$^{\dagger}$ & 4.212$^{\rm w}$ \\
APARENT / D3LM & 1.200 & 33.774 & 10.311 & 58.430 & \textbf{2.563--2.579} & 18.290--18.436 & 4.105 & 0.977$^{\rm w}$ \\
K562 / D3LM & 0.917$^{\dagger}$/1.205 & \textbf{1.536$^{\dagger}$/2.013} & 6.593$^{\dagger}$ & 4.542$^{\dagger}$/5.108 & 31.926$^{\dagger}$ & 19.861$^{\dagger}$ & 21.745$^{\dagger}$ & 0.985$^{\rm w}$ \\
Protein-1YCR / EvoDiff & 36.390 & \textbf{38.790} & 292.733 & 38.831 & 39.588 & 49.397 & 59.609 & 16.484$^{\rm f}$ \\
\bottomrule
\end{tabular}
\end{table}

\subsection{What does and does not need a convergence theorem}

The backward pass is a finite acyclic recursion, not an iterative fixed-point
algorithm: after exactly $L$ levels it has computed the partition and all
suffix masses on the retained graph.  Accordingly, its relevant theorem is
exactness and finite complexity, not asymptotic convergence.

For a fixed selected feature design, the scorer objective in
Eq.~(\ref{eq:scorer-objective}) is convex.  This follows because binary logistic
loss is convex in its affine logit, while the $\ell_1$ and squared $\ell_2$
penalties are convex.  If both target classes occur and the token and AC
coefficient blocks have positive $\ell_2$ penalties, the objective is coercive
and has at least one finite minimizer; it is strongly convex in the regularized
coefficient subspace.  These facts justify a global rather than local fitting
target.

The implementation reports numerical optimizer termination, iterations, and
held-out metrics. We do not claim a new optimizer rate for its mixed
$\ell_1/\ell_2$ objective; such a rate would require a specified proximal
solver and optimality residual, or a smooth reformulation. Likewise, carrier
training and the frozen diffusion model are empirical estimation procedures;
their diagnostics do not imply convergence to a global population optimum.
Tab.~\ref{tab:convergence-boundaries} summarizes these claim boundaries.

\begin{table}[h]
\centering
\small
\setlength{\tabcolsep}{4pt}
\caption{Claim boundaries for ``convergence.''}
\label{tab:convergence-boundaries}
\begin{tabular}{p{0.25\linewidth}p{0.28\linewidth}p{0.31\linewidth}}
\toprule
Component & Formal statement & Empirical evidence \\
\midrule
Backward DP & finite exact recursion & graph/support sizes \\
Ancestral replay & exact at $\tau_a=1$ & sampling rule and seed \\
Scorer fit & convex fixed-design objective & solver status and validation metrics \\
Carrier fit & no global-optimum claim & epoch/loss/convergence flag \\
Diffusion host & frozen & checkpoint and schedule \\
\bottomrule
\end{tabular}
\end{table}

\section{Carrier and Objective Construction}
\label{app:reward-training}

\subsection{Carrier assets and scope}

The carrier is part of every query, but its realization is asset specific. A
structured HMM or Markov carrier contributes target-independent sequence
statistics; a one-state identity carrier is the neutral finite-state instance.
Non-identity carriers follow the same target-free CoDD-style training
procedure~\citep{li2026codd},
with asset-specific data, state size, block length, and frozen checkpoint. The
identity carrier requires no training. The LLaDA language artifact is trained
on WikiText-103~\citep{merity2017pointer}, has 64 latent states, uses block
length 32, and occupies 30.88 MiB when serialized. It is frozen before
task-specific objective fitting, while Dream, D3LM, and Protein use their own
frozen carrier configurations.

\subsection{Mining, selection, and fitting}

Alg.~\ref{alg:reward-construction} separates reusable corpus processing
from target-dependent statistics.  Mining never reads labels; changing the
target reuses its candidate set and count matrices and reruns only selection
and fitting.

\paragraph{Objective-agnostic candidate grammar.}
For training rows $x^{(1)},\ldots,x^{(N)}$, the document frequency of token
pattern $p$ is
\[
\operatorname{df}(p)=\sum_{n=1}^{N}\mathbf{1}[p\text{ occurs in }x^{(n)}].
\]
The miner enumerates literal token $n$-grams of orders one through four and
keeps a pattern only when
$m_{|p|}\leq\operatorname{df}(p)\leq\rho N$.  The default thresholds are
$(m_1,m_2,m_3,m_4)=(8,5,3,3)$ and $\rho=0.70$; deterministic ordering by
decreasing document frequency, order, and token tuple applies the
80,000-candidate resource cap.  Mining reads neither the target labels nor
validation rows.  The stored sparse matrix uses occurrence counts
\[
X_{nj}=c_j(x^{(n)}),
\]
not binary presence, so training and inference agree when a pattern repeats or
overlaps itself.
\begin{algorithm}[h]
\caption{Build and compile a count-logit reward}
\label{alg:reward-construction}
\begin{algorithmic}[1]
\Require tokenized training text; training labels; validation text and labels;
document-frequency limits; feature orders $1{:}4$
\Ensure signed token weights, weighted phrase automaton, validation report, and
one decision record per candidate
\State $\mathcal{C} \gets \Call{MineNgrams}{\text{training text},1{:}4,
\text{frequency limits}}$
\State $X_{\mathrm{train}},X_{\mathrm{val}} \gets
\Call{CountOccurrences}{\mathcal{C},\text{training text},\text{validation text}}$
\Comment{no labels used above}
\State $\mathcal{S}\gets\emptyset$
\For{$n=1,\ldots,4$}
  \For{candidate $c\in\mathcal{C}$ of order $n$}
    \State $d_c\gets\Call{ResidualizedLogisticTest}{X_{\mathrm{train}}[:,c],
    \text{training labels},\mathcal{S}}$
  \EndFor
  \State apply within-order false discovery rate (FDR) correction to all $d_c$
  \State $\mathcal{A}_n\gets\Call{ApplyGatesAndCaps}{\{d_c:|c|=n\},\mathcal{S}}$
  \Comment{stability, artifact, heredity, and child-over-parent gates}
  \State $\mathcal{S}\gets\mathcal{S}\cup\mathcal{A}_n$
\EndFor
\State $(b,u,w)\gets\Call{FitRegularizedLogistic}{X_{\mathrm{train}}[:,\mathcal{S}],
\text{training labels}}$
\State report $\gets\Call{EvaluateOnly}{b,u,w,X_{\mathrm{val}},
\text{validation labels}}$
\State automaton $\gets\Call{CompileAhoCorasick}{\mathcal{S},w}$
\State \Return $(b,u,\text{automaton},\text{report},\{d_c:c\in\mathcal{C}\})$
\end{algorithmic}
\end{algorithm}

\paragraph{Residualized short-to-long score test.}
At feature order $k$, let $H_k$ contain an intercept, five label-free text
diagnostics (log length, unique-token fraction, repeated-bigram fraction,
maximum-run fraction, and punctuation fraction), and the features accepted at
orders below $k$.  A class-weighted ridge logistic null fit gives probabilities
$p$, residual $r=\omega\odot(y-p)$, and
$D=\operatorname{diag}(\omega\odot p\odot(1-p))$.  For candidate count column
$c_j$, the residualized score and information are
\begin{align}
U_j &= c_j^\top r,\\
I_j &= c_j^\top D c_j
-c_j^\top D H_k(H_k^\top D H_k+\gamma P)^{-1}H_k^\top D c_j,\\
Z_j &= U_j/\sqrt{I_j},
\label{eq:feature-score-test}
\end{align}
where $P$ leaves the intercept unpenalized.  We convert $|Z_j|$ to a two-sided
normal $p$-value; $\gamma$ is the null-fit ridge strength.  We apply
Benjamini--Hochberg correction within each order
\citep{benjamini1995controlling}.

A candidate then passes five transparent gates: corrected $q\leq0.05$ and
$|Z_j|\geq2$; absolute correlation at most $0.98$ with every text diagnostic;
at least one selected prefix or suffix parent for orders above one; at least a
$2\%$ gain over the strongest same-sign parent; and matching sign plus
significance in at least four of five stratified $70\%$ training subsamples.
Survivors are ordered by corrected $q$, $|Z|$, document frequency, and token
tuple before per-order and global resource caps are applied.  The accepted
columns are appended to $H_k$ before testing order $k+1$; the readable
reference defaults cap each order at 64 and the complete selection at 256.
Every candidate has a decision record containing its statistic and first
rejection reason.
Training and inference both count every occurrence, including overlaps.
Validation data report the area under the receiver operating characteristic
curve (AUC), balanced accuracy, and log loss but never enter the fit or these
gates.

We use these gates as a reproducible screening procedure. The paper does not
claim that the final, adaptively constructed dictionary has a finite-sample
5\% false-discovery-rate guarantee: such a claim would additionally require
the calibration and dependence assumptions of the individual score tests and
the sequential selection procedure.

With binary training targets $y_n$, token coefficients $u$, selected-pattern
coefficients $w$, and count-semantic feature values $c_j$, the fitted scorer
uses binary cross-entropy (BCE),
\[
R_\phi(x)=b+\sum_i u_{x_i}+\sum_j w_jc_j(x)
\]
and minimizes
\begin{align}
\mathcal{L}(\phi)={}&\frac{1}{N}\sum_{n=1}^{N}
\operatorname{BCE}\!\left(y_n,\sigma(R_\phi(x^{(n)}))\right)
\nonumber\\
&+\alpha_u\lVert u\rVert_2^2
+\alpha_1\lVert w\rVert_1
+\alpha_2\lVert w\rVert_2^2.
\label{eq:scorer-objective}
\end{align}
The sigmoid belongs to the supervised fit and to diagnostic probabilities.
The generation-time factor uses the signed energy $R_\phi$ itself, as stated
in Eq.~(\ref{eq:guided-law}).

Eq.~\ref{eq:scorer-objective} is the logistic special case, not a universal
training law. RTP uses a backbone- and cohort-bound binary or soft-toxicity
asset. Malinois uses a target-highest binary objective with reverse-complement
tied sixmer features. DeepSTARR and APARENT use their frozen asset-specific
regression objectives. The exact feature grammar, labels, loss, regularization,
and coefficients are task specific rather than inferred from the task name.
CommonGen trains no scorer: its terminal
$0/-\infty$ DFA is fixed by each record. Offline scorer accuracy is diagnostic;
only frozen generation followed by the independent task evaluator supports a
result claim.

\subsection{Task-specific controller construction}
\label{app:task-controllers}

The inference engine sees only the finite-state interface, but the compiler
and scored text scope are task specific.  Tab.~\ref{tab:task-controllers}
states the complete mapping.  The learned controllers are weighted
deterministic automata, not hard accept/reject DFAs; CommonGen is the terminal
hard-DFA instance.

\begin{table}[H]
\centering
\footnotesize
\setlength{\tabcolsep}{4pt}
\renewcommand{\arraystretch}{1.08}
\caption{How each task becomes a controller. Literal feature mining is shared
in form but labels and fitted weights are task specific. CommonGen, Dyck, and
Sudoku have no learned count-logit reward. Protein bigrams are chronological
action-order features, not adjacent residues in the final sequence.}
\label{tab:task-controllers}
\begin{tabularx}{\textwidth}{@{}>{\raggedright\arraybackslash}p{0.13\linewidth}>{\raggedright\arraybackslash}p{0.18\linewidth}>{\raggedright\arraybackslash}X>{\raggedright\arraybackslash}p{0.21\linewidth}@{}}
\toprule
Task & Scored scope & Automatic controller construction & Terminal rule \\
\midrule
RTP & generated continuation only & Mine and select train-only count features for
$\mathbf{1}[\text{continuation toxicity}<0.2]$; compile signed token terms and
selected phrase weights into one weighted AC machine. & zero for every state \\
\shortstack[l]{Common\\Gen} & generated sentence & For each record, enumerate standalone base,
regular-plural, and possessive forms; tokenize them and compile the reachable
lexical-prefix state plus a satisfied-concept bit mask. & $0$ iff every required
concept is complete and no forbidden concept is hit; $-\infty$ otherwise \\
Dyck repair & 12 locked prefix tokens plus 20 repaired suffix tokens & Compile a
bounded-stack graph, intersect forward/backward min-plus optima, then attach
frozen evidence and carrier weights. & empty stack at length 32 and exact
minimum substitution distance \\
DeepSTARR & complete 48-token DNA sequence & Compile the frozen order-4
occurrence-count energy into token and automaton terms. & zero for every state \\
APARENT & complete 48-token DNA sequence & Compile the frozen token-unigram
regression energy. & zero for every state \\
Malinois & complete 48-token DNA sequence & Compile the target-specific
reverse-complement-tied sixmer/token-unigram energy. & zero for every state \\
Sudoku & editable cells with clues locked & Enumerate legal completions under
the standard row, column, and box rules, then build a layered joint-support
graph over the scheduled cells. The objective state records the retained
completion prefix. & one iff the terminal prefix is a declared legal
completion; zero otherwise \\
Protein-1YCR & chronological editable-action sequence with motif locks &
Compile 20 unigram, 400 action-order bigram, and 320 action-position weights
fitted on partial action prefixes; replay uses the same editable action order
and full 20-amino-acid support. & zero for every state \\
\bottomrule
\end{tabularx}
\end{table}

\begin{algorithm}[h]
\caption{Compile a task-specific controller}
\label{alg:task-controller}
\begin{algorithmic}[1]
\Require task $t$; task record; tokenizer; declared candidate support;
selected patterns and fitted weights when $t$ is learned
\Ensure token energies plus a finite-state controller
\If{$t\in\{\mathrm{RTP},\mathrm{DeepSTARR},\mathrm{APARENT},\mathrm{Malinois}\}$}
  \State place every fitted unigram weight in the token-energy table
  \State insert every selected phrase and its signed weight into a token trie
  \State add failure links and propagate terminal weights along those links
  \State \Return token energies and the resulting weighted AC controller
\EndIf
\If{$t=\mathrm{CommonGen}$}
  \For{concept $c$ in the record}
    \State $V_c\gets\Call{StandaloneForms}{c,\text{base, plural, possessive}}$
    \State $T_c\gets\Call{TokenizeWithBoundaries}{V_c,\text{tokenizer}}$
  \EndFor
  \State build a prefix matcher over $\bigcup_c T_c$
  \State initial DFA state $\gets$(empty active prefixes, zero concept mask,
  not rejected)
  \State breadth-first enumerate transitions on the declared support; update
  active prefixes, completed-concept bits, and the reject flag
  \State terminal score $\gets0$ iff all required bits are set and not rejected;
  $-\infty$ otherwise
  \State \Return zero token energies and the reachable hard DFA
\EndIf
\State \Return \Call{UnsupportedTask}{t}
\end{algorithmic}
\end{algorithm}

For learned tasks, fixed context is consumed before the scored scope to obtain
the carrier and automaton boundary state, and rewards are accumulated only on
the declared continuation or biological output sequence. For CommonGen, word
boundary handling belongs to the tokenizer adapter. The same accepted token
strings define both generation and exact evaluation. This
prevents evaluator-only morphology rules from silently differing from the DFA
used during generation.

\section{Experimental Setup and Evaluation}
\label{app:data-evaluation}

Tab.~\ref{tab:dataset-contracts} summarizes the evaluation contracts.

\begin{table*}[h]
\centering
\small
\setlength{\tabcolsep}{2pt}
\caption{Task and evaluator contracts.  The scorer used for guidance is not the
sole evaluator.}
\label{tab:dataset-contracts}
\begin{tabular}{p{0.12\linewidth}p{0.20\linewidth}p{0.17\linewidth}p{0.20\linewidth}p{0.22\linewidth}}
\toprule
Task & Generation problem & Target / constraint & Primary control metric & Quality and audit metrics \\
\midrule
RTP & prompt continuation & non-toxic continuation & ToxicBERT continuation toxicity and satisfaction & prompt-conditioned GPT-2-large corpus PPL \\
\shortstack[l]{Common\\Gen} & concept realization & terminal lexical DFA acceptance & fail-closed all-concept satisfaction & exact/normalized coverage; forbidden hits; unreachable support; conditional PPL \\
Dyck & fixed-length repair & valid minimum-edit completion & exact validity & all-output Hamming edit and legal-edge rate \\
Sudoku & clue-preserving completion & valid $9\times9$ grid & exact validity & clue preservation and output format \\
DeepSTARR & unconditional DNA generation & developmental activity & official DeepSTARR H5 predictor & native masked LOO and failure count \\
APARENT & unconditional DNA generation & isoform usage & official APARENT H5 predictor & native masked LOO and diversity diagnostics \\
Malinois & cell-type-specific DNA generation & target score exceeds both non-target scores & target-highest rate & target margin, native masked LOO, and failure count \\
Protein-1YCR & motif-scaffold completion & low motif RMSD after folding & Success@1\AA{} & mean motif RMSD, evaluator coverage, and model/oracle work \\
\bottomrule
\end{tabular}
\end{table*}

The split law is uniform: fitted components use TRAIN only, hyperparameters use
DEV only, and a frozen configuration is evaluated once on its declared final
panel. CommonGen uses 993 DEV records and 1,497 disjoint TEST records. Dyck uses
1,024 DEV inputs and 9,798 evaluation inputs. Sudoku uses 64 development
puzzles from PRISM training solutions and its released 2,000-puzzle hard
validation file~\citep{kim2025fine}, as redistributed by
MDM-VGB~\citep{jeon2026mdmvgb}; we do not call it an independently certified
test split.

The final panels are fixed benchmark cohorts. Some were evaluated earlier in
this study and are therefore not new prospective holdouts. For each reported
system, fitting uses TRAIN data and configuration selection uses DEV data; no
final-panel metric selects its configuration. The stated source disjointness
concerns these experimental records and does not certify that the pretrained
backbone never encountered similar data.

RTP is a shared stress benchmark rather than a population estimate: its 1,024
prompts are selected by toxicity of a frozen diffusion-only continuation, and
the identical ordered prompt IDs are then used for every method. No guided
method defines its own panel. DNA experiments use 1,000 unconditional draws per
displayed system. These are independent frozen-system samples with a common
backbone, output contract, predictor, and sample count; they are not paired
molecules and do not imply unseen biological generalization.

\subsection{Metric computation}
\label{app:metric-computation}

For $N$ outputs, hard satisfaction is $N^{-1}\sum_n\mathbf 1[V(x^{(n)})]$.
Sudoku validity $V$ requires correct output format, clue preservation, and all
row, column, and box rules. Its separate clue score is the fraction of outputs
preserving every given clue. CommonGen validity requires every requested concept
under the accepted-form and word-boundary rules. Dyck validity checks bracket
balance and type matching, while mean edit distance is the symbol-level Hamming
distance to the input averaged over all outputs, including invalid ones.
Minimum-edit optimality is verified separately against the exact repair cost.

RTP non-toxicity is $N^{-1}\sum_n\mathbf 1[t_n<0.5]$, where $t_n$ is the
generated continuation's \nolinkurl{unitary/toxic-bert@4d6c22e74ba2}
score~\citep{hanu2020detoxify}. Prompt-conditioned
\nolinkurl{openai-community/gpt2-large@32b71b12589c} corpus
perplexity~\citep{radford2019language} is
$\exp(\sum_n\mathrm{NLL}_n/\sum_n T_n)$, where NLL denotes negative
log-likelihood and $T_n$ counts scored output tokens. The prompt supplies
context but contributes neither loss nor tokens to that denominator. Corpus
PPL aggregates token losses before exponentiation.

For DNA, the declared task crop is scored by the official predictor.
DeepSTARR and APARENT report the mean target output over successfully scored
sequences, with scoring failures recorded separately. DeepSTARR Pass@95 is the
fraction of all generated sequences reaching the fixed activity threshold. Malinois target-highest rate
is the fraction whose target-cell score strictly exceeds both other cell
scores. Native D3LM masked LOO masks one token at a time, records the original
token's negative log probability with all other tokens visible, averages those
losses over all scored tokens, and exponentiates. It is a pseudo-perplexity,
not a normalized joint likelihood or a target-activity metric.

Protein motif RMSD compares the declared motif atoms in the OmegaFold structure
with the reference under the evaluator's alignment. Mean RMSD averages these
per-output distances. Success@1\AA{} is
$N^{-1}\sum_n\mathbf 1[\mathrm{RMSD}_n\leq1\,\text{\AA}]$.
Satisfaction, isoform usage, target-highest, and success rates are displayed as
percentages in the main tables.

The remaining setup details are organized in
App.~\ref{app:prompt-contracts} (inputs and prompts),
App.~\ref{app:baselines} (baseline adaptations), and
App.~\ref{app:experimental-protocol} (generation and selection).

\subsection{Exact generation prompt contracts}
\label{app:prompt-contracts}

We freeze prompts as part of the experimental condition.  Diffusion-only,
carrier-only, learned guidance, and other algorithmic baselines receive the
\emph{formal task input} below.  The prompt-only baseline receives the same
record, start canvas, and generation budget, but applies the corresponding
\emph{explicit-control wrapper}.  Prompt-only is therefore a named baseline,
not a point in the guided-temperature sweep. Braces below denote fields filled
from the frozen dataset record. Free continuations begin after the displayed
suffix.

\subsubsection{RealToxicityPrompts (RTP)}

For diffusion-only and all algorithmic guidance methods, the model receives
the dataset prefix verbatim:

\begin{quote}
\small\ttfamily
\{prefix\}
\end{quote}

The matched prompt-only baseline uses the frozen
\texttt{rtp\_prompt\_only\_v2} template:

\begin{quote}
\small\ttfamily\sloppy
Continue the text naturally and coherently using non-toxic language. Output\par
only new text and do not repeat the given prefix.\\[0.6em]
Text: \{prefix\}
\end{quote}

Only the newly generated continuation is passed to the toxicity evaluator;
conditional perplexity is computed with the original prefix as context.

\subsubsection{CommonGen}

Every method receives the same concept-conditioned task prompt:

\begin{quote}
\small\ttfamily\sloppy
Write one short sentence using all these words: \{concept$_1$\},
\ldots, \{concept$_m$\}.\\
Sentence:
\end{quote}

This prompt states the benchmark input rather than an additional soft-control
instruction.  We therefore do not report a separate prompt-only CommonGen
baseline.  Exact control is defined by the record-conditioned terminal DFA,
and acceptance requires every requested concept in the generated sentence.

\subsubsection{Structured and biological inputs}

Dyck, Sudoku, and the biological tasks use structured records rather than an
additional natural-language control prompt. Dyck methods receive the same
32-token repair record with the first 12 tokens locked; Sudoku methods receive
the same puzzle and clue mask. DNA methods share the D3LM backbone, 48-token
output contract, target identity, and evaluator but draw independent
unconditional samples. Thus language and symbolic inputs are row-matched,
whereas the biological comparison is distributional rather than molecule-paired.
For the 48-token DNA contract, Tab.~\ref{tab:biology} includes a
$b=48$ full-budget anchor with independently frozen method-specific DEV
selections. This uses the same masked host and records 48 total token commits
per trajectory; the native schedule can still contain zero- or multi-token
steps. BoN and particle methods retain their additional trajectory costs.
The confirmation scope, retained fixed-panel evaluations, and mode collapse
are summarized in App.~\ref{app:complete-results}. The main table displays the
K562 $b=48$ target-specific result. HepG2 and SK-N-SH are not part of the main
endpoint comparison and are therefore omitted from that table; HepG2 appears
only in the mechanism diagnostic in Fig.~\ref{fig:guidance-mechanism}A.
Language Full uses a 32-token chunk at $b=32$, and Dyck may finish its 20
editable positions within that allocation. Sudoku Full uses $b=64$ and permits
at most one new cell per update, with 10--53 realized calls on the panel in
Tab.~\ref{tab:sudoku-multisolution}.

\section{Baselines and Hyperparameter Selection}
\label{app:baselines}

Algorithmic sources are CDD~\citep{cardei2025cdd}, CDM~\citep{kim2026cdm},
D-CBG~\citep{schiff2025simple}, DG/TAG~\citep{nisonoff2025guidance},
GILC~\citep{dou2026gilc}, and MDM-VGB~\citep{jeon2026mdmvgb}.
Tab.~\ref{tab:baseline-contracts} summarizes each port.

\begin{table*}[h]
\centering
\scriptsize
\setlength{\tabcolsep}{3pt}
\caption{Executable baseline recipes. ``Port'' means our implementation of the
published algorithmic idea on the frozen host backbone; it does not claim
author-code parity.}
\label{tab:baseline-contracts}
\begin{tabular}{p{0.12\linewidth}p{0.15\linewidth}p{0.24\linewidth}p{0.24\linewidth}p{0.17\linewidth}}
\toprule
Method & Implementation & Fitted object and supervision & Online rule & Reported compute \\
\midrule
Base & host implementation & none & native reverse diffusion & $b$ reverse updates \\
CDD & same-backbone port & train-only projection expert & top-$k$ projection with augmented-Lagrangian inner updates & $b$ plus inner-loop count $I$ \\
CDM & sequential Monte Carlo (SMC) port & train-only twist & $P$ particles with fixed proposal, effective sample size (ESS) threshold, and resampling law & task-dependent particle work; see below \\
D-CBG & same-backbone port & train-only corrupted-state classifier & classifier reweights token alternatives; DNA uses the declared full-vocabulary Taylor form & $b$ updates and candidate width \\
DG-TAG & method-inspired port & frozen task predictor & rescore top-$k$ successor states and update logits & $b$ updates and predictor calls \\
GILC-DB & adapted port & train-only task reward model & Monte-Carlo logit correction at every reverse update & $b$ reverse updates; MC samples reported separately \\
MDM-VGB & task-specific same-backbone ports & train-only verifier & LLaDA/Dyck uses native reveal/remask backtracking; other tasks retain their documented forward-value adaptations & $b$ denoising steps, with realized search moves and verifier calls disclosed separately \\
\method{} & proposed method & train-only additive reward plus target-free carrier & exact backward messages and ancestral replay on the retained product graph & $b$ backbone updates plus DP FLOPs \\
\bottomrule
\end{tabular}
\end{table*}

For the new multi-solution Sudoku adaptation, CDM particles operate inside
exact feasible-completion guidance. Its denoising-step budget counts host
updates, while particle work remains a separate compute coordinate.
Full-trajectory particle methods, including
DNA and Protein CDM, instead incur $Pb$ backbone row-forwards.

For conditional language and symbolic tasks, methods receive identical ordered
row IDs, checkpoint, prompt or structured record, output length, evaluator, and
reverse-step budget. Biological systems use independent unconditional draws
under a common backbone, output contract, predictor, and sample count; we do
not describe them as paired. Particle count, MC samples, search width,
best-of-$N$, and realized NFE remain separate compute coordinates rather than
being hidden inside $b$. The allocated denoising budget, realized committed-token
count, and realized backbone NFE are distinct quantities. A global 1:1 point denotes
token-complete masked denoising under the same host; no experiment here changes
the host into a block-diffusion model.

To contextualize the denoising-budget comparison, Tab.~\ref{tab:illustrative-flops}
estimates backbone arithmetic for Base and \method{} on the same 1,497
LLaDA/CommonGen prompts in separate efficiency runs.
\begin{table}[!htbp]
\centering
\small
\setlength{\tabcolsep}{5pt}
\caption{Estimated backbone TFLOPs per output in the LLaDA/CommonGen
efficiency runs. We use $F\approx2PS$ per forward, nominal
$P=8\times10^9$, $b$ observed forwards, and mean padded lengths $S=54.50$
(Base) and $53.20$ (\method{}). The length difference reflects batching.
They exclude \method{}'s dynamic programming and are neither profiled nor
total FLOPs; other tasks require their own sequence lengths and model-call
counts.}
\label{tab:illustrative-flops}
\begin{tabular}{@{}rcc@{}}
\toprule
$b$ & Base & \method{} \\
\midrule
$8$ & 6.976 & 6.809 \\
$32$ & 27.903 & 27.238 \\
\bottomrule
\end{tabular}
\end{table}

\subsection{Shared protocol}
\label{app:experimental-protocol}

\subsubsection{Frozen model and inference settings}

Tab.~\ref{tab:runtime-contract} gives the language defaults before task-specific selection.

\begin{table}[h]
\centering
\small
\setlength{\tabcolsep}{4pt}
\caption{Default COFFEE language configuration before benchmark-specific DEV
selection. Temperatures and $\lambda$ are independent controls.}
\label{tab:runtime-contract}
\begin{tabular}{lr}
\toprule
Field & COFFEE language default \\
\midrule
Diffusion checkpoint & LLaDA-8B-Base \\
Generation / block length & 32 / 32 \\
Denoising budget & 32 \\
Evidence temperature $\tau_e$ & 1.0 \\
Carrier-emission temperature $\tau_m$ & 1.0 \\
Ancestral temperature $\tau_a$ & 1.0 \\
Final host sampling temperature & 0.0 \\
Reward strength $\lambda$ & 1.0 \\
Initial / maximum token top-$k$ & 16 / 64 \\
Other-mass expansion threshold & 0.55 \\
State beam & 32 \\
Guided fraction & all denoising updates \\
Carrier states / vocabulary & 64 / 126,346 \\
\bottomrule
\end{tabular}
\end{table}

This is a language default, not a universal recipe. Dream, D3LM,
Protein, and task-specific controller settings follow the development-time
search and selection rules in Tabs.~\ref{tab:appendix-selection-contracts}
and~\ref{tab:appendix-search-axes}. The carrier construction is described in
App.~\ref{app:reward-training}.

\subsubsection{Selection and statistical reporting}

Selection is performed independently inside each denoising budget. Language
tasks first maximize the declared control metric subject to the same-budget
quality rule, then use the deterministic tie-break in
Tab.~\ref{tab:appendix-selection-contracts}. Dyck and Sudoku select by exact
validity before secondary edit or clue metrics. Biological ports select on the
declared DEV predictor while retaining native LOO as a separately reported
trade-off; no TEST quantity selects a configuration.

For CommonGen and RTP, the common soft-language quality rule compares
prompt-conditioned corpus PPL on DEV with the unguided Base for the same
backbone, task, and budget:
\[
\mathrm{PPL}^{\mathrm{DEV}}_{\mathrm{candidate}}
\leq 1.5\,\mathrm{PPL}^{\mathrm{DEV}}_{\mathrm{Base}}.
\]
Among candidates satisfying this common cap, selection maximizes the task's
control metric. If none satisfies it, the lowest-PPL point on the DEV frontier
is frozen without relaxing the cap. Its final-panel control and quality
measurements are still reported for matched systems, but the point cannot
support a claim that requires satisfying the cap. PPL is a quality proxy and
does not by itself rule out repetitive or degenerate outputs.

Every displayed cell is one frozen system evaluated on its declared panel.
Success rates use all planned outputs as their denominator. A continuous
quantity that is undefined for a failed output is summarized only over
evaluable outputs. The main table reports claim-facing metrics, while
App.~\ref{app:complete-results} records only the failure or coverage
qualifications needed to interpret those claims. A dash denotes an unavailable
or inapplicable measurement rather than zero.
Paired intervals are used only when systems share the same input rows and a
paired resampling unit; independent DNA draws are not assigned paired
intervals. Small numerical differences on unpaired panels are treated as
descriptive rather than as significance claims.

\begin{table}[!htbp]
\centering
\small
\setlength{\tabcolsep}{3pt}
\caption{Selection contracts for the results discussed in the main paper.
No final-panel metric is used to choose a configuration.}
\label{tab:appendix-selection-contracts}
\begin{tabularx}{\textwidth}{@{}p{0.12\linewidth}p{0.17\linewidth}p{0.20\linewidth}X@{}}
\toprule
Benchmark & Development data & Final panel & Frozen selection rule and interpretation \\
\midrule
CommonGen & 993 records, with 992 in the corrected common pool & 1,497 disjoint records & Maximize exact lexical acceptance under the common PPL cap, followed by coverage, PPL, and strength tie breaks. \\
RTP & Method-independent development panel & Common 1,024-prompt stress panel selected by Base toxicity & Maximize non-toxicity under the same-budget quality rule. Results describe this stress panel rather than population-average RTP behavior. \\
Dyck & 1,024 inputs & 9,798 disjoint inputs & Maximize exact validity, then minimize all-output edit distance. Independent minimum-repair verification is reported only for the audited configurations described in Sec.~\ref{app:dyck-repair-diagnostic}. \\
Sudoku & 64 puzzles constructed from released training solutions & Released 2,000-puzzle hard panel & Maximize validity, then clue preservation, format, model work, and grid order. The upstream panel is named validation and is not described as an independently certified test set. \\
DNA & Task-specific frozen development systems & 1,000 outputs per frozen system & Select by the declared predictor on development data. Native LOO and diversity remain separate diagnostics. Comparisons are distributional rather than molecule paired. \\
Protein-1YCR & Common development geometries & Common 400-geometry panel & Maximize Success@1\AA{}, then minimize mean motif RMSD and evaluator work. Particle methods retain their additional model and evaluator calls. \\
\bottomrule
\end{tabularx}
\end{table}

\begin{table}[!htbp]
\centering
\small
\setlength{\tabcolsep}{3pt}
\caption{Development-time search axes for the methods shown in the main
comparison. Batch size is an execution choice rather than a selected scientific
parameter. Particle, candidate, and verifier work remain separate from the
denoising-step budget.}
\label{tab:appendix-search-axes}
\begin{tabularx}{\textwidth}{@{}p{0.14\linewidth}X p{0.30\linewidth}@{}}
\toprule
Method & Parameters searched on development data & Computation fixed before final evaluation \\
\midrule
COFFEE & Language evidence, carrier, and ancestral temperatures in $\{0.3,0.6,0.9\}$, followed by reward strengths in $\{0.5,1,2,4,8\}$ for declared representatives. Biological scale and temperature are selected separately for each task. & Carrier, compiled objective, retained support rule, and backbone checkpoint. \\
CDD & Projection strength $\eta\in\{0.2,0.4,0.6,0.8\}$, initial multiplier $\mu_0\in\{1,5,10\}$, and inner iterations $I\in\{10,25,50\}$ where applicable. & Projection expert, candidate width, and outer update schedule. \\
CDM & Twist strength $\beta\in\{0.1,0.2,2\}$ for language, with task-specific biological refinements. & Particle count, proposal, effective-sample-size threshold, and resampling rule. \\
D-CBG & Classifier strength $\gamma\in\{0.5,1,2\}$, with separately selected full-vocabulary scales for Malinois. & Corrupted-state classifier and candidate rule. \\
DG-TAG & Rescoring strength $\gamma\in\{1,5,20\}$, plus the declared $0.1$ CommonGen refinement. & Frozen task predictor and top-$k$ successor set. \\
GILC-DB & Language strength $\beta\in\{0.1,0.3,1,3,10\}$, with task-specific biological scale refinements. & Reward model, Monte Carlo sample count, and candidate width. \\
MDM-VGB & Step-matched language and symbolic strength $\lambda\in\{0,0.05,2,8\}$ and biological strength $\lambda\in\{0.5,2,8\}$. & Frozen verifier and the host schedule capped at the reported budget. \\
\bottomrule
\end{tabularx}
\end{table}

\section{Main-Claim Support and Result Qualifications}
\label{app:complete-results}

Tab.~\ref{tab:main-results} and Fig.~\ref{fig:budget-primary} contain the
claim-facing numerical results. The appendix records only the additional
qualifications needed to interpret those results rather than repeating every
numerical cell already visible in the table and budget curves.

All endpoints in Tab.~\ref{tab:main-results} are frozen final-panel outcomes
and remain numerically ranked on the displayed metrics. Mode collapse and
quality-gate failures remain disclosed through the reported quality and
diversity diagnostics rather than turning those TEST outcomes into report-only
results.

A hollow marker in Fig.~\ref{fig:budget-primary} denotes a point outside the
standard comparison set. The map below lists every such point. Report-only
points are lower-budget archival diagnostics excluded from standard ranking,
while K562 points are separately qualified. No oracle, reference, or
quality-cap-failed point is plotted.

\begingroup
\scriptsize
\setlength{\parindent}{0pt}
\noindent\textit{Report-only points.} CommonGen / Dream, CDD (4), D-CBG (4, 16); RTP / Dream, CDD (16), D-CBG (4), COFFEE (4, 16); DeepSTARR, D-CBG (4), GILC-DB (4, 16); CommonGen / LLaDA, CDD (4), D-CBG (4, 16); RTP / LLaDA, CDD (4, 16), D-CBG (4).\par
\noindent\textit{Separately qualified points.} K562, D-CBG (16), GILC-DB (16).\par
\endgroup

Tab.~\ref{tab:main-claim-support} summarizes the main claim boundaries.

\begin{table}[H]
\centering
\footnotesize
\setlength{\tabcolsep}{3pt}
\caption{Additional evidence and claim boundaries for the results interpreted
in the main text.}
\label{tab:main-claim-support}
\begin{tabularx}{\textwidth}{@{}p{0.13\linewidth}p{0.33\linewidth}X@{}}
\toprule
Main-text result & Additional evidence & Required interpretation \\
\midrule
Dyck repair & The compiler constructs the minimum-substitution valid support before probabilistic replay. On the two independently audited full-panel files, every returned output matches an independently recomputed optimum. & The independent audit covers the stated Dream $b=4$ and LLaDA $b=32$ configurations. It is not a blanket empirical certificate for every budget. \\
Multi-solution Sudoku & COFFEE and the support-aware CDM adaptation preserve validity across the displayed budgets. Sec.~\ref{app:sudoku-joint-diagnostic} separately compares coherent joint replay with independent exact marginals. & Solver-assisted adapters share exact feasible-completion information. Their results test how parallel choices are coupled, not whether one method discovered the Sudoku solution set without a solver. \\
CommonGen and RTP & CommonGen uses a record-conditioned hard automaton. RTP uses a learned soft objective on a common 1,024-prompt stress panel. & Full lexical acceptance and toxicity reduction answer different questions. RTP values are stress-panel measurements and not population estimates. \\
Regulatory DNA & K562 has 32 distinct token sequences among 1,000 COFFEE outputs. DeepSTARR produces one distinct all-T sequence despite favorable activity and native LOO. & Predictor scores and native LOO do not establish diversity, biological function, or experimental efficacy. DNA systems are independent frozen draws rather than paired molecules. \\
Protein-1YCR & Success@1\AA{} and mean motif RMSD are computed on the common 400-geometry panel. & Both metrics use the same OmegaFold evaluation family, one motif, and project adaptations. They are not independent wet-lab validation. \\
\bottomrule
\end{tabularx}
\end{table}

\section{Canonical Algorithms and Failure Semantics}
\label{app:algorithm}

\subsection{The fixed-step compiled update}
\label{app:compiled-update}
The following pseudocode gives the fixed-step inference and sampling
procedure used in Sec.~\ref{sec:backward-dp}. Its factors are fixed
before the backward pass; support and numerical implementation details
are discussed in the surrounding appendix.

\begin{algorithm}[H]
\caption{One \method-guided diffusion update}
\label{alg:coffee-update}
\footnotesize
\begin{algorithmic}[1]
\Require current sequence $x_t$; frozen denoiser; carrier factors $K_{1:L}$;
compiled objective $(\delta_{C,1:L},\psi_{C,1:L},\eta_C)$; host reverse transition
$q_{\mathrm{host}}$
\Ensure next diffusion state $x_{t-1}$
\State Build the clamped unary evidence $e_{t,1:L}$ using
Eq.~\ref{eq:evidence}.
\State Set $M_L(h,a)\gets\eta_C(a)$ for every terminal product state $(h,a)$.
\For{$i=L,L-1,\ldots,1$}
  \State Compute $M_{i-1}$ from $M_i$ using
  Eq.~\ref{eq:product-recurrence}.
\EndFor
\If{$M_0(h_0,a_{\mathrm{init}})=0$}
  \State \Return \Call{NoPositiveMass}{retained support}
\EndIf
\State $(h,a)\gets(h_0,a_{\mathrm{init}})$
\For{$i=1,2,\ldots,L$}
  \State Sample $(u,h')$ with probability proportional to
  \Statex \hspace{1em}$e_{t,i}(u)K_i(h,u,h')\psi_{C,i}(a,u)$
  \Statex \hspace{1em}$\cdot M_i\!\left(h',\delta_{C,i}(a,u)\right)$.
  \State $\widetilde x_0^i\gets u$; \quad
  $(h,a)\gets(h',\delta_{C,i}(a,u))$.
\EndFor
\State Sample $x_{t-1}\sim q_{\mathrm{host}}(x_{t-1}\mid x_t,\widetilde x_0)$
\Statex \hspace{1em}using the host model's original proposal-to-state rule.
\State \Return $x_{t-1}$
\end{algorithmic}
\end{algorithm}

\begin{algorithm}[H]
\caption{Strict minimum-edit repair with frozen evidence}
\label{alg:strict-repair}
\small
\begin{algorithmic}[1]
\Require reference $y$; locked positions $F$; per-position candidates
$C_{1:L}$; bounded-language transition $\delta$; frozen evidence; target-free
carrier; random seed
\Ensure one globally minimum-edit valid repair or a fail-closed record
\State initialize forward min-plus cost $D_0(a_0)=0$ and $D_0(a)=\infty$ otherwise
\For{$i=1,\ldots,L$}
  \State $D_i(a')\gets\min_{a,v:\delta(a,v)=a'}
  [D_{i-1}(a)+\mathbf{1}[v\neq y_i]]$, respecting $F$ and $C_i$
\EndFor
\State run the symmetric backward min-plus recursion to obtain suffix costs
\State $d^\star\gets$ minimum terminal forward cost
\For{every layered edge $(i,a,v,a')$}
  \State retain the edge iff prefix cost $+$ edit cost $+$ suffix cost equals
  $d^\star$
\EndFor
\If{the retained graph has no accepting terminal path}
  \State \Return \Call{NoFeasibleRepair}{declared support}
\EndIf
\State attach frozen evidence and carrier scores to every retained edge
\State run sum-product backward messages on the retained graph
\If{the root message $M_0(h_0,a_0)=0$}
  \State \Return \Call{NoPositiveMass}{retained support}
\EndIf
\State $x\gets\Call{AncestralReplay}{\text{messages},\text{seed}}$
\State \Return $(x,\{i:x_i\neq y_i\},d^\star)$
\end{algorithmic}
\end{algorithm}

The min-plus stage is score free: it determines the exact feasible edit
support. The later sum-product stage changes only the relative probability of
tied minimum-edit repairs. A resource limit returns a separate resource-failure
record rather than dropping states as an unreported beam or being mislabeled
as mathematical infeasibility.

\paragraph{Log-domain implementation.}
The production backend stores $\mathsf m_i(s)=\log M_i(s)$ and replaces every
sum in the backward recurrence by log-sum-exp. At unit ancestral temperature,
the replay logit for edge $e:s\to s'$ is $g_i(e)+\mathsf m_i(s')$; this is
Eq.~\ref{eq:appendix-edge-conditional} in log space and does not introduce
a second message named $B$. Compatible rows are batched without changing the
sampling law.

\paragraph{Failure records.}
The implementation distinguishes an unsatisfiable hard objective on the
declared support, zero positive carrier--objective mass on an otherwise
feasible retained graph, a resource-limit abort during graph construction,
and a numerical failure. Only the first two are mathematical support outcomes;
resource and numerical failures are reported separately and never relabeled as
infeasibility.

\section{Guidance Dynamics and Diagnostic Experiments}
\label{app:qualitative}

\subsection{Multi-solution Sudoku joint consistency}
\label{app:sudoku-joint-diagnostic}
The TRAIN-derived 120-puzzle diagnostic uses one source-disjoint seed and
fixed settings. Joint replay, independent exact marginals, and the joint
reference share the same 64-bit floating-point (FP64) conditional scores,
decoder, and confidence rule; D-CBG uses a different feasibility-derived
confidence source. Joint replay, uniform joint sampling, and the reference all
reach 120/120, while independent marginals fail at smaller budgets. This
supports the need for coherent joint sampling under parallel commits; because
uniform joint sampling also succeeds, it does not establish a validity gain
from learned carrier weighting.
\begin{table}[H]\centering\small
\setlength{\tabcolsep}{5pt}
\caption{Multi-solution Sudoku: joint consistency under parallel commitments.
Entries count valid, clue-preserving outputs out of 120 puzzles; 40 puzzles
have 2, 3--8, and 9--32 solutions, respectively. This is a TRAIN-derived,
source-disjoint exploratory panel with one seed and fixed diagnostic settings,
not the 2,000-puzzle evaluation or a complete-baseline comparison.
Full ($b=64$) permits at most one new cell per step, with 10--53 actual calls.
The FP64 joint reference and independent-marginal row use identical conditional
scores, support, temperature, decoder, and confidence rule; proposals may lead
to different subsequent commit positions. Uniform joint still uses model-based
commit ordering.}
\label{tab:sudoku-multisolution}
\begin{tabular}{@{}lccccc@{}}
\toprule
Method & $b=4$ & $b=8$ & $b=16$ & $b=32$ & \textbf{Full ($b=64$)}\\
\midrule
Joint replay & 120 & 120 & 120 & 120 & 120\\
Independent exact marginals & 68 & 100 & 114 & 118 & 120\\
Uniform joint & 120 & 120 & 120 & 120 & 120\\
D-CBG & 13 & 40 & 71 & 113 & 120\\
\midrule
Joint reference & 120 & 120 & 120 & 120 & 120\\
\midrule
Maximum cells committed per step & 14 & 7 & 4 & 2 & 1\\
\bottomrule
\end{tabular}
\end{table}

\subsection{Minimum-edit repair development diagnostic}
\label{app:dyck-repair-diagnostic}
The diagnostic fixes a 12-token valid source prefix, repairs a 20-token suffix,
and uses four bracket tokens with a bounded Dyck stack. On 9,877 invalid
development leaves, every returned sequence is valid and globally minimum-edit
within the declared support; mean edit distance is 8.283 and masked-suffix
pseudo-perplexity is 12.27. A matched MDM-VGB port repairs 28.41\% of rows;
its valid subset has mean edit distance 10.57 and pseudo-perplexity 10.21.
These are shared-LLaDA DEV diagnostics, not held-out evidence or a reproduction
of the original MDM-VGB experiment.

\subsection{Guidance dynamics and frozen-query diagnostics}
\label{app:guidance-diagnostics}

Panel A of Fig.~\ref{fig:guidance-mechanism} uses 64 shared HepG2 IDs under
the same identity carrier, host schedule, and retained-support rule. Its
ribbons count transitions of each rollout's complete clean proposal between
displayed updates; they are not probability flow. The scorer-on and
$W_C\equiv1$ arms therefore describe proposal trajectories under two frozen
systems rather than an intervention that isolates cross-position lookahead.

Panels B--C use nine distinct frozen K562 $b=4$ queries. For a fixed
conditioning state $z=(h,a,x_0^{<i})$, the Full and Local action
distributions share denoiser evidence, carrier factors, current objective
weight, and retained support. The Full conditional uses the target-weighted
continuation message $M_i$, whereas Local replaces only that message with
the neutral continuation mass $B_i$. Panel B selects one query at position
25 and one prefix product state, then varies the future objective horizon
$d$: the objective factors at positions $i+1,\ldots,i+d$ are included,
while those after $i+d$ are set to one. The entire 48-token lattice,
including suffix denoiser evidence, carrier factors, candidate support, and
automaton state transitions, remains intact. Thus $d=0$ reproduces Local
and $d=23$ reproduces Full for the same fixed prefix. The intermediate
points are controlled recalculations, not diffusion-time updates. In this
query, the probability of \texttt{AAAAAA} is 0.771 at $d=0$ or 1,
0.643 at $d=2$, and 0.500 from $d=3$ through 23. Positions 27 and 28
are observed \texttt{AAAAAA} tokens in the frozen state; enabling their
objective factors changes the current action odds by factors
$\exp(-0.623)$ and $\exp(-0.588)$, respectively. Their product is 0.298.

For panel C, let $\mu_i^0(z)$ be the normalized occupancy under exact
neutral-prefix replay. The position-level statistic is
\begin{equation}
\overline{\operatorname{TV}}_i
=\sum_z\mu_i^0(z)\frac12\sum_u
\left|p_{\mathrm{Full}}(u\mid z)-p_{\mathrm{Local}}(u\mid z)\right|.
\label{eq:k562-prefix-tv}
\end{equation}
Panel C plots the empirical CDF of this prefix-conditional TV averaged over
neutral-prefix occupancy; it does not first mix the two action distributions
and then compute TV. The highlighted position is the largest variation among
these nine queries and was selected only for illustration. Its
occupancy-weighted TV is 0.250, whereas panel B uses one selected prefix with
conditional TV 0.270. The displayed 19/224 count is descriptive. Here future
refers to positions not yet sampled in the same clean reconstruction, rather
than later diffusion steps; this diagnostic establishes no external K562
benefit.

Qualitative examples are included only when the matched diffusion-only output,
strongest eligible baseline output, guided output, independent evaluator score,
and internal scorer trace are all available for the same record. Examples do
not substitute for aggregate results and are not selected by TEST score alone.

\subsection{CommonGen completion weighting and joint sampling}
\label{app:commongen-components}

This diagnostic reuses 256 official validation records selected before the
component outcomes by concept count and tokenized concept length. Each arm uses
the same records and three seeds with LLaDA-8B-Base, 32 generated tokens,
denoising budget $b=8$, and the same prompt and retained-support rule. Evidence,
carrier-emission, and ancestral temperatures are $(0.9,0.3,1.0)$. The
target-free carrier is either the fitted 64-state WikiText HMM (H) or a
one-state identity carrier (I). We recompile each record's lexical DFA after
generation and compute prompt-conditioned GPT-2-large PPL over the generated
continuations. All 5,376 outputs are nonempty and included in both metrics.

The comparisons keep the carrier and the lexical controller fixed within each
carrier setting. At a fixed query, let $g_0(v\mid s)$ denote the neutral action
conditional and $h(v,s)$ the neutral mass of accepting completions after action
$v$. Full assigns weight proportional to $g_0(v\mid s)h(v,s)$. Feasible keeps
the neutral weight $g_0(v\mid s)$ only when $h(v,s)>0$. It therefore removes
dead ends without using the relative mass of the remaining completions.
Neutral removes the terminal acceptance weight while retaining the same
candidate graph. Full-marginal computes the exact single-token marginals of
Full on the retained support at each fixed query, then draws the concurrently
committed tokens independently. Once different tokens are committed, later
denoiser queries and their supports may diverge.

\begin{table}[H]
\centering
\small
\caption{CommonGen component diagnostic on 256 reused validation inputs and
three seeds per arm. PPL is the prompt-conditioned GPT-2-large corpus PPL over
generated continuations. No output is excluded from scoring.}
\label{tab:commongen-components}
\begin{tabular}{lrr}
\toprule
Carrier and sampling rule & Lexical coverage (\%) & PPL $\downarrow$ \\
\midrule
H--Neutral & 22.01 & 16.81 \\
H--Feasible & 100.00 & 20.28 \\
H--Full & 100.00 & 18.93 \\
H--Full-marginal & 83.98 & 18.34 \\
I--Neutral & 30.99 & 19.56 \\
I--Feasible & 100.00 & 23.71 \\
I--Full & 100.00 & 21.97 \\
\bottomrule
\end{tabular}
\end{table}

H--Full and H--Full-marginal differ in lexical coverage by $16.02$ percentage
points, with a 95\% input-cluster bootstrap interval of $[13.15,18.75]$.
Both Full and Feasible attain 100\% coverage under H and I. Relative to
Feasible, Full reduces corpus token NLL by $0.0691$ $[0.0221,0.1148]$ under H
and $0.0764$ $[0.0319,0.1235]$ under I; brackets again give 95\% input-cluster
intervals. Each of 5,000 bootstrap draws resamples the 256 inputs while keeping
their three seed outputs together. The joint-versus-marginal result concerns lexical
coverage, not a PPL gain. These are development diagnostics on reused
validation records, not fresh held-out results. Unit ancestral temperature
makes the joint and marginal laws directly comparable but differs from the
main-table selected configurations. Their fixed-query exactness is conditional
on retained candidate support, up to floating-point error;
PPL measures an external language-model score rather than sentence quality.

\section{Reproducibility and Evidence Provenance}
\label{app:reproducibility}

Each displayed value binds one task, backbone, cohort, method, budget, metric,
and frozen configuration; values are not pooled across these contracts.
Reproduction constructs the data roles in
Tab.~\ref{tab:appendix-selection-contracts}, trains or loads the TRAIN-only
component, runs the DEV grids in Tab.~\ref{tab:appendix-search-axes}, applies
the deterministic selection rule, and evaluates the frozen configuration once.
Algs.~\ref{alg:reward-construction}--\ref{alg:strict-repair} specify the
computations. Failed outputs remain in the denominator, and comparisons require
shared task, split, prompt, evaluator, support, and budget contracts.

\section{Related Work}
\label{sec:related}

\paragraph{Discrete diffusion and dependence modeling.}
D3PMs generalize discrete diffusion beyond uniform corruption and include an
absorbing-state construction that connects diffusion with masked generation
\citep{austin2021structured}. Discrete Copula Diffusion supplements denoiser
predictions with a separately trained copula model to restore dependencies
between output variables and reduce the number of denoising steps
\citep{liu2025discretecopula}. Chain-structured conditional random fields give
a standard graphical-model treatment of conditional sequence dependence and
dynamic-programming inference \citep{sutton2012introduction}. These works
motivate the use of structured dependence models inside discrete generation,
while \method{} separately represents the target-free dependence carrier and
the sequence objective used for guidance.

\paragraph{Guidance for discrete diffusion.}
Existing methods steer discrete diffusion through intermediate predictors,
continuous relaxations, or repeated evaluation of candidate outputs.
Predictor-based methods estimate properties at noisy states and use those
estimates to modify the reverse transitions
\citep{nisonoff2025guidance,schiff2025simple}, while NOS optimizes continuous
denoiser representations \citep{gruver2023protein}. SVDD evaluates the future
value of candidate clean predictions \citep{li2025svdd}, and CSMC constructs a
Metropolis--Hastings chain over complete clean samples
\citep{phunyaphibarn2026csmc}. These approaches support broad classes of
objectives, but require either a reliable signal at intermediate noise levels
or repeated reward evaluation during sampling. Training such predictors also
requires task-specific examples across the corruption process, where the
attribute may be difficult to identify at high noise. \method{} instead starts
from a clean-sequence objective and compiles it into finite-state factors. This
avoids a separate neural attribute predictor at every noise level when the
objective admits a compact compiled representation.

\paragraph{Concurrent inference-time steering.}
Several concurrent methods steer frozen diffusion models through learned
twists, local prediction corrections, or scheduling policies. CDM amortizes a
twisted sequential Monte Carlo proposal with a learned twist
\citep{kim2026cdm}, while TRI-TSMC iteratively fits the twist through
trust-region updates \citep{wang2026tritsmc}. GILC modifies clean-prediction
logits using reward information \citep{dou2026gilc}; DLM-SWAI applies
precomputed token-level attribute biases \citep{an2026dlmswai}; and SAD steers
the denoiser away from unsafe regions \citep{yusuf2026sad}. DPRM instead uses a
Doob-transform process reward to change token ordering
\citep{bu2026dprm}, while MDM-VGB adds verifier-guided remasking and
backtracking \citep{jeon2026mdmvgb}. These approaches support broad objectives
through learned twists, local corrections, ordering, or search. \method{}
targets objectives that admit compact finite-state factors and computes their
continuation weights on the retained product graph.

\paragraph{Search for biological sequence design.}
Tree-guided diffusion has also been used for biological design. MP2D combines
conditional diffusion, constrained Monte Carlo tree search, and iterative
refinement for multi-objective protein design \citep{kong2026mp2d}.
DNA-CRAFT combines class-conditioned diffusion with Monte Carlo tree guidance
for cell-type-specific regulatory DNA design \citep{awasthi2026dnacraft}.
These methods explicitly search multiple denoising trajectories under external
objectives, whereas \method{} performs exact backward inference for the
compiled objective within each retained fixed-step graph.

\paragraph{Future-aware control through tractable completion models.}
Autoregressive control has likewise used predictions or tractable models of
future continuations. FUDGE trains a discriminator to predict whether a prefix
will eventually satisfy an attribute \citep{yang2021fudge}. GeLaTo distills a
hidden Markov model for tractable conditioning on lexical constraints
\citep{zhang2023gelato}, while Ctrl-G pairs such a model with a finite-state
logical constraint \citep{zhang2024ctrlg}. TRACE combines a distilled hidden
Markov model with a lightweight attribute model to aggregate future attribute
probabilities without sampling the continuations individually
\citep{weng2025trace}. \method{} extends this completion-based view to a
diffusion step in which unresolved variables may occur anywhere in the
sequence and must be treated as one joint clean reconstruction. It also
compiles count-additive soft preferences into weighted state transitions, so
the same backward-message algorithm supports both hard acceptance and soft
reward.

\paragraph{Compiled representations for structured inference.}
Knowledge compilation studies representations that make otherwise expensive
queries tractable \citep{darwiche2002knowledge}. A probabilistic circuit is a
directed computation graph that represents a distribution through sums and
products. Under appropriate structural compatibility conditions, circuit
operations such as products and marginalization remain tractable
\citep{vergari2021atlas}. Semantic Probabilistic Layers use this structure to
combine neural predictions with logical constraints \citep{ahmed2022spl}, and
Neurosymbolic Diffusion Models use discrete diffusion to capture dependencies
among symbols inside a neuro-symbolic predictor
\citep{vanKrieken2025neurosymbolic}. CoDD introduces a tractable probabilistic
layer that restores cross-position dependence to factorized diffusion-language
predictions \citep{li2026codd}. For sequence objectives, weighted automata
provide a finite-state algebra \citep{mohri2009weighted}, while
Aho--Corasick automata retain the suffix information needed to count
overlapping pattern occurrences \citep{aho1975efficient}. Recent constrained
diffusion methods use dynamic programming or finite automata to enforce hard
regular constraints \citep{suresh2025dingo,dang2026automata}. \method{}
composes a target-free dependence carrier with a separately compiled objective
and derives one product-state message-passing and sampling algorithm for hard
and soft sequence guidance.

\end{document}